\documentclass{article}
\newif\ifmypreprint

\mypreprintfalse

\newcommand{\preprintonly}[1]{%
  \ifmypreprint
    #1%
  \fi
}

\newcommand{\journalonly}[1]{%
  \ifmypreprint
  \else
    #1%
  \fi
}
\ifmypreprint
  \PassOptionsToPackage{preprint}{neurips_2026}
\fi
\usepackage[preprint]{neurips_2026}

\usepackage[symbol]{footmisc}
\usepackage[utf8]{inputenc} 
\usepackage[T1]{fontenc}    
\usepackage{hyperref}       
\usepackage{url}            
\usepackage{booktabs}       
\usepackage{amsfonts}       
\usepackage{nicefrac}       
\usepackage{microtype}      
\usepackage{xcolor}         

\usepackage{blindtext}
\usepackage{amsmath}
\usepackage{amsthm}
\usepackage{algorithmic, algorithm}
\usepackage{graphicx}
\usepackage{subcaption}
\usepackage{amssymb}
\usepackage{dsfont}
\usepackage{wrapfig}
\usepackage{setspace}
\usepackage[all]{hypcap}

\hypersetup{
  colorlinks=true,
  linktoc=all,
  citecolor=blue,
  urlcolor=blue,
  linkcolor=black
}

\theoremstyle{plain}
\newtheorem{theorem}{Theorem}[section]
\newtheorem{lemma}{Lemma}[section]
\newtheorem{proposition}[theorem]{Proposition}
\newtheorem{corollary}[theorem]{Corollary}
\theoremstyle{definition}
\newtheorem{definition}[theorem]{Definition}
\newtheorem{assumption}[theorem]{Assumption}
\theoremstyle{remark}
\newtheorem{remark}[theorem]{Remark}
\DeclareMathOperator*{\argmin}{arg\,min}
\newcommand{\RETURN}{\STATE \textbf{return}}

\usepackage{cleveref}
\crefname{assumption}{assumption}{assumptions}
\Crefname{assumption}{Assumption}{Assumptions}

\crefname{proposition}{proposition}{propositions}
\Crefname{proposition}{Proposition}{Propositions}

\crefname{definition}{definition}{definitions}
\Crefname{definition}{Definition}{Definitions}

\crefname{appendix}{appendix}{appendices}
\Crefname{appendix}{Appendix}{Appendices}

\crefname{section}{section}{sections}
\Crefname{section}{Section}{Sections}

\title{Group-Aware Matrix Estimation and Latent Subspace Recovery}

\author{%
  Hamza Golubovic\thanks{These authors contributed equally.} \\
  Department of Statistics,\\
  Columbia University\\
  New York, NY 10027 \\
  \texttt{hg2723@columbia.edu} \\
  \And 
  Matthew Shen\footnotemark[1] \\
  Department of Statistics,\\
  Columbia University\\
  New York, NY 10027 \\
  \texttt{ms7079@columbia.edu} \\
  \AND
  Genevera I. Allen \\
  Department of Statistics, \\Irving Institute for Cellular Dynamics, \\Zuckerman Institute\\
  Columbia University\\
  New York, NY 10027 \\
  \texttt{genevera.allen@columbia.edu} \\
  \And
  Tarek M. Zikry\thanks{Corresponding author.} \\
  School of Data and Information Sciences,\\
  University of North Carolina at Chapel Hill\\
  Chapel Hill, NC 27599\\
  \texttt{tarek@unc.edu} 
}

\begin{document}

\maketitle

\begin{abstract}
Modern matrix completion problems often involve heterogeneous data whose rows simultaneously belong to many meta-categories, such as demographic and age groups in recommendation systems, or region and recording session labels in neural electrophysiological experiments. Standard low-rank estimators impose a single global latent geometry, which can recover average structure but may smooth away subgroup-specific variation, especially when observations are unevenly distributed across groups. We introduce Group-Aware Matrix Estimation (GAME), a convex estimator for overlapping subgroup-wise low-rank matrix estimation. GAME regularizes category-specific submatrices through overlapping nuclear-norm penalties, allowing related groups to borrow information while preserving local latent structure in a shared coordinate system. We provide finite-sample guarantees for both reconstruction error and subgroup-specific subspace recovery, showing how performance depends on sampling density, subgroup rank, and overlap structure. Experiments on synthetic, recommendation, ecological, and neuroscience datasets show that GAME is most beneficial in structured missingness regimes, where subgroup-aware regularization improves both reconstruction accuracy and latent subspace fidelity. Across these benchmarks, GAME is competitive or best among global low-rank, side-information, and modern imputation baselines, with the largest gains when subgroups exhibit distinct low-rank structure.
\end{abstract}
\section{Introduction}

In recommendation systems, users may belong to overlapping demographic groups such as age, gender, and occupation \citep{harper2015movielens}. In biological, experimental, and longitudinal studies, samples may be indexed by subject, batch, session, spatial region, or other partially overlapping sources of variation \citep{10.1093/biomtc/ujad002, CHEN2024676}. In such settings, the goal is often not only to reconstruct missing entries accurately on average, but also to preserve subgroup-specific latent structure that may be scientifically or operationally important. Standard nuclear-norm estimators impose a single global low-rank penalty \citep{candes2008exactmatrixcompletionconvex}, which can smooth away subgroup-specific variation - especially for groups that are unevenly or weakly observed, where faithful estimation often matters most. More broadly, heterogeneous data need not conform to one shared low-rank geometry; different subgroups may exhibit local low-dimensional structure that is related through overlap but not identical \citep{candes2008exactmatrixcompletionconvex}. This motivates an overlap-aware estimator that shares information across groups while preserving local structure.



In this work, we propose Group-Aware Matrix Estimation (GAME), a convex framework for heterogeneous matrix estimation with overlapping meta-categories. 
GAME (Figure \ref{fig:fig1}) balances faithful reconstruction of the observed entries with nuclear-norm regularization on meta-category-defined submatrices, allowing local low-rank structure to be learned within groups while overlap propagates information across them. 
This yields a single coherent estimate that borrows strength across groups without forcing all groups into the same latent subspace.

\begin{figure}[H]
  \centering
  \includegraphics[width=0.8\linewidth]{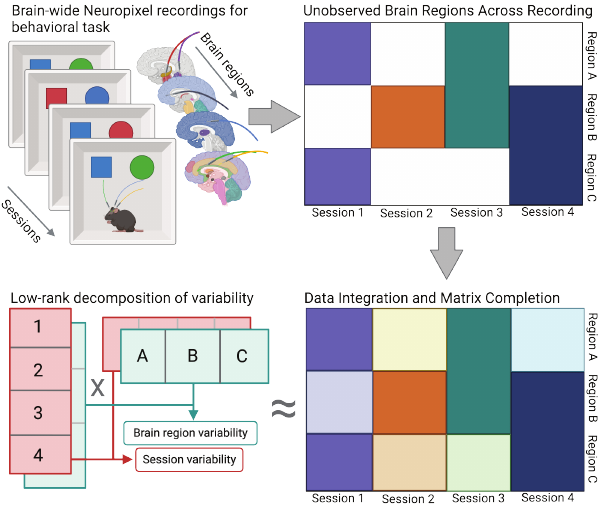}
  \caption{\textbf{Overview of motivating group-aware matrix estimation in neuroscience.} Brain-wide Neuroxpixel recordings measure single-neuron temporal responses to stimuli in mice and monkeys, but can only measure a small number of brain regions at once, leaving large block-missingness across recording sessions and brain regions. We develop Group Aware Matrix Estimation (GAME) to impute missing region dynamics and decompose variability across meta-categories such as region and session.}
  \label{fig:fig1}
\end{figure}
GAME is a general estimator for overlapping subgroup-wise low-rank structure; it is not restricted to settings with an explicit block-sampling model. However, its inductive bias is most beneficial when subgroups exhibit distinct low-rank structure. Structured missingness is one important instance: when some subgroups are weakly observed, overlap allows underrepresented groups to borrow information without being smoothed into a single global estimate.

To optimize the GAME objective at scale, we use the proximal average framework \citep{doi:10.1137/070687542, NIPS2013_49182f81}, which avoids the auxiliary-variable and consensus overhead of ADMM in high-category settings \citep{boyd2011admm}. Beyond scalable optimization, GAME also supports analysis of subgroup-specific latent geometry: by regularizing each subgroup through an overlapping nuclear norm, the estimator supports both reconstruction and local subspace recovery.

We evaluate GAME on synthetic data with hierarchical group structure, MovieLens-100k \citep{harper2015movielens} under global and block-wise missingness and meta-category misspecification, multi-label birdsong classification from the BirdSet benchmark \citep{mary_clapp_2023_7525805, rauch2025birdsetlargescaledatasetaudio}, and a scientific case study on multi-session Neuropixel recordings \citep{CHEN2024676}. Across these settings, GAME performs competitively or best across all settings, with the largest margins when subgroups exhibit distinct low-rank structure or when the goal is to recover local latent geometry.

\paragraph{Contributions}
Together, these ideas lead to four main contributions:
\begin{itemize}
    \item \textbf{A subgroup-aware convex estimator for heterogeneous matrix estimation.} We introduce GAME, which promotes overlapping subgroup-wise low-rank structure while maintaining a shared coordinate system across the full matrix.
    
    \item \textbf{Scalable optimization without consensus overhead.} We develop a proximal-averaging solver for the GAME objective that is practical in settings with many overlapping categories.
    
    \item \textbf{Provable reconstruction and subspace fidelity under overlap.} We provide finite-sample guarantees for both reconstruction error and subgroup-specific subspace recovery, clarifying how performance depends on sampling, rank, and overlap structure.
    
    \item \textbf{Empirical validation on synthetic and real datasets.} We show that GAME performs competitively or better than benchmark methods across diverse settings, with the largest gains in structured missingness regimes. On Neuropixels data, GAME recovers region-specific latent dynamics, illustrating value beyond aggregate imputation accuracy.
    
\end{itemize}

The remainder of the paper develops the formulation, optimization, and theory, followed by experiments on synthetic and real data.
\section{Related Work}

\subsection{Matrix Completion with Side Information}
 
Low-rank matrix completion has been extensively studied as a framework for recovering low-rank matrices from partial observations. Classical formulations rely on the assumption that the target matrix is globally low-rank and can be recovered via a convex relaxation using the nuclear norm \citep{candes2008exactmatrixcompletionconvex, Recht_2010}. Non-convex approaches, such as alternating least squares (ALS) \citep{hastie2014matrixcompletionlowranksvd, JMLR:v11:mazumder10a} have also been widely adopted for their computational efficiency on high-dimensional matrices. While these models enjoy strong theoretical guarantees under incoherence and random sampling assumptions \citep{candes2010tightoracleboundslowrank, negahban2011restrictedstrongconvexityweighted}, they treat all rows and columns equally, and do not incorporate meta (i.e.\ ``side'') information that may be readily available in real-world applications.
 
Several methods incorporate \textit{side information}—row or column features—into the completion problem. Note that \textit{side information} is analogous to \textit{meta-category}. Inductive Matrix Completion (IMC) \citep{jain2013provableinductivematrixcompletion} assumes that the target matrix admits a factorization of the form 
$$\mathbf{X} \approx \mathbf{A M B}^\top,$$
where $\mathbf A$ and $\mathbf B$ are known feature matrices and $\mathbf M$ is a low-rank parameter matrix to be learned. The feature matrices may contain one-hot encodings of categorical information or singular vectors of a cheaply completed matrix approximation of $\mathbf{X}$. Similar to IMC, Maxide \citep{NIPS2013_e58cc5ca} exploits side information by restricting the target matrix to lie in the subspaces spanned by known feature matrices, but penalizes the completion problem with a lower-dimensional trace-norm, yielding substantial computational and sample-complexity gains when the side information is well-aligned with the latent structure. DirtyIMC \citep{NIPS2015_0609154f} addresses noisy or incomplete side information by augmenting IMC with a low-rank residual term, while FNNM \citep{yang2022feature} instead decomposes the target into a globally low-rank component and a sparse component tied to the side information, balancing global and local structure via nuclear norm and $\ell_1$ regularization jointly. \citep{pmlr-v267-chang25e} observe that when side information is incomplete, the component of the target matrix lying outside the available feature subspaces lacks any inherent low-rank structure and must be explicitly regularized, proposing to penalize the nuclear norm of this orthogonal complement projection alongside the global nuclear norm.

\subsection{Issues with Side Information Methods}
A common theme across these side-information methods is that each imposes a specific \emph{generative} assumption about how features produce entries: a bilinear factorization through known feature matrices in IMC and Maxide, augmented by a low-rank residual in DirtyIMC, a sparse correction in FNNM, or an orthogonal-complement penalty in \citep{chang2024lowrankcovariancecompletiongraph}. When this commitment is misspecified, the estimator converges not to the truth but to its projection onto the assumed model class. In other words a side-information method given misleading features is biased toward a wrong target, whereas standard matrix completion given no features simply returns the best low-rank fit. This concern is sharpened in the categorical setting: encoding group membership as one-hot row features collapses a bilinear model $\mathbf{X} \approx \mathbf{A M B}^\top$ to a per-group mean model, in which every row in a category is identical up to noise. The natural workaround is supplementing the one-hot encoding with leading singular vectors of a cheaply imputed approximation of $\mathbf{X}$. However, it quickly becomes unclear how many singular vectors to retain for each category, becoming an unprincipled hyperparameter without theoretical guarantees. 

These limitations motivate our GAME approach, which makes no commitment about data generation, robust even when side-information is weak/noisy. Categorical meta-information is used not to model how entries arise from features, but only to define a partition of the rows of $\mathbf{X}$ under which each block is hypothesized to be approximately low-rank (a purely observational condition). When this hypothesis is misaligned with the data, the truth still lies in the model class (in particular, the globally low-rank case is recoverable through the interpretable $\lambda_c$), so the estimator has at most a tuning inefficiency rather than the structural bias incurred by a misspecified bilinear model.
\preprintonly{
\begin{figure}
  \centering
  \includegraphics[width=0.8\linewidth]{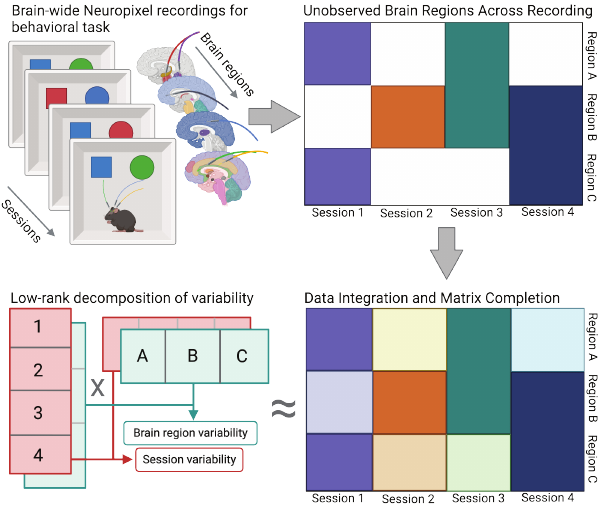}
  \caption{\textbf{Overview of motivating group-aware matrix estimation in neuroscience.} Brain-wide Neuroxpixel recordings measure single-neuron temporal responses to stimuli in mice and monkeys, but can only measure a small number of brain regions at once, leaving large block-missingness across recording sessions and brain regions. We develop Group Aware Matrix Estimation (GAME) to impute missing region dynamics and decompose variability across meta-categories such as region and session.}
  \label{fig:fig1}
\end{figure}
}
\subsection{Low-rank estimation and latent structure in heterogeneous data}

Low-rank and latent-structure methods are widely used when high-dimensional observations admit a more compact representation. In some heterogeneous settings, this structure may be better described by multiple low-dimensional components rather than a single shared subspace. For example, subspace clustering and low-rank representation methods explicitly model data as arising from a union of low-dimensional subspaces \citep{elhamifar2013sparse,Liu_2013}. In related multi-task settings, shared low-dimensional representations can improve learning across related tasks \citep{argyriou2008convex, obozinski2011grouplassooverlapslatent}. More broadly, recent work in high-dimensional biology has emphasized that jointly modeling structured variation across multiple sample sets can be more powerful and scientifically informative than analyzing each dataset separately \citep{10.1093/biomtc/ujad002}.

These perspectives are relevant when the goal is not only reconstruction, but also recovering latent structure that organizes heterogeneous observations. In such settings, the latent representation can support downstream tasks such as segmentation, integration across cohorts, or prediction from partially observed data. This setting arises when different groups are weakly observed but still share enough structure to support information transfer.

Neural population recordings provide one important instance of this broader problem (Figure \ref{fig:fig1}. Recent methods use low-rank or latent-variable structure to denoise, compress, and interpret large-scale neural data. Penalized matrix decomposition can separate low-dimensional signal from noise in functional imaging data \citep{buchanan2018penalizedmatrixdecompositiondenoising}, while tensor decompositions reveal structure across temporal, spatial, and experimental axes \citep{pellegrino2024TCA,Williams2018TCA}. Latent factor models such as GPFA and LFADS instead impose explicit temporal structure on neural dynamics \citep{Yu2009GPFA,Pandarinath2018LFADS}, and post-hoc alignment methods compare low-dimensional dynamics across conditions or animals \citep{dabagia2022comparinghighdimensionalneuralrecordings}. Recent multi-session and multi-animal models further show that shared latent structure can improve inference or prediction across partially observed recordings \citep{azabou2024multi, zhang2024towards, zhang2025exploiting, xia2025inpaintingneuralpictureinferring}. However, these approaches typically assume a single shared latent space or require aligned recordings across sessions, and do not regularize overlapping subgroup-specific low-rank structure within a unified matrix estimation framework.
\section{Group-Aware Matrix Estimation (GAME)}
We introduce Group-Aware Matrix Estimation (GAME), a convex estimator designed to leverage overlapping categorical structures, promoting group-wise low-rank reconstructions while preserving a shared coordinate system. Unlike global estimators that smooth out distinct latent geometries, GAME allows information to be shared across categories while controlling complexity locally via category-specific penalties.

\subsection{Problem Formulation}

We consider a partially observed matrix $\mathbf{X} \in \mathbb{R}^{n \times m}$, with rows as observational units and columns as shared features. Each row $i \in [n]$ may belong to multiple meta-categories $c \in \mathcal{C}$, such as an individual being both 35 years of age and female sex in recommendation systems. Let $\Omega \subseteq [n] \times [m]$ denote the observed entries and $\mathcal{P}_{\Omega}$ the associated projection operator. The meta-categories in $\mathcal{C}$ may overlap: a single row can belong to several categories simultaneously, so the matrix cannot be decomposed into independent blocks.

Rather than assuming that $\mathbf{X}$ admits a single global low-rank explanation, GAME posits that each meta-category $c \in \mathcal{C}$ induces meaningful low-rank structure on its corresponding submatrix, even though these structures overlap and may not align perfectly across categories. The GAME objective, defined in Equation \ref{eq:game}, is to estimate a single matrix $\mathbf{W}$ that preserves these category-specific patterns while coherently reconciling them:
\begin{equation}
\label{eq:game}
\widehat{\mathbf{W}}
\in
\argmin_{\mathbf{W} \in \mathbb{R}^{n \times m}}
\frac{1}{2}
\big\|
\mathcal{P}_{\Omega}(\mathbf{X} - \mathbf{W})
\big\|_F^2
+
\lambda
\sum_{c \in \mathcal{C}}
\alpha_c
\,
\| \mathbf{W}_c \|_*,
\end{equation}
where $\mathbf{W}_c \in \mathbb{R}^{n_c \times m}$ denotes the submatrix of $\mathbf{W}$ corresponding to rows in meta-category $c$, $\| \cdot \|_*$ is the nuclear norm, $\lambda > 0$ is a global regularization parameter, and $\{ \alpha_c \}_{c \in \mathcal{C}}$ are nonnegative weights satisfying $\sum_{c \in \mathcal{C}} \alpha_c = 1$.

Standard matrix completion assumes that all rows share a common rowspace, whose principal right singular vectors explain the low-rank behavior of every row. But rows belonging to different meta-categories may exhibit partially orthogonal latent structure, for example distinct age groups or sexes in a recommendation system. Partitioning the regularization by meta-category allows GAME to respect these distinctions, using category membership as a proxy for differing low-rank behavior.

\subsection{Algorithm and Optimization}

Since the meta-categories overlap, the GAME regularizer is non-separable, making consensus-based methods such as ADMM \citep{glowinski1975approximation, gabay1976dual, boyd2011admm} impractical due to the scaling of auxiliary variables with $|\mathcal{C}|$.

\setlength{\intextsep}{8pt}
\setlength{\columnsep}{12pt}
\begin{wrapfigure}{r}{0.6\textwidth}
\vspace{-1.0em}

\begin{minipage}{0.6\textwidth}
\small
\renewcommand{\arraystretch}{1.6}
\setlength{\itemsep}{0.35em} 

\begin{algorithm}[H]
\caption{GAME Algorithm (PA-APG)}
\label{algorithm:game}
\vspace{0.25em}
\begin{algorithmic}[1]
\STATE Initialize $\mathbf{W}^0 = \mathbf{Y}^1$, $\gamma>0$, $\eta_1 = 1$
\vspace{0.25em}

\FOR{$k = 1,2,\dots$}
    \vspace{0.35em}
  \STATE $\mathbf{Z}^k = \mathbf{Y}^k - \gamma\,\mathcal{P}_\Omega(\mathbf{Y}^k - \mathbf{X})$
  \vspace{0.25em}

  \STATE $\mathbf{D}_c \mathbf{Z}^k = \mathbf{U}_c^k \mathbf{\Sigma}_c^k {\mathbf{V}_c^k}^\top$
  \textbf{for} $c = 1,\dots,|\mathcal{C}|$ \textbf{(SVD)}
  \vspace{0.25em}

  \STATE
  $\begin{aligned}
  \mathbf{W}^k
  &=
  \sum_{c\in\mathcal{C}} \alpha_c \Big[
      (\mathbf{I} - \mathbf{D}_c^\top \mathbf{D}_c)\mathbf{Z}^k \\[-0.25em]
  &\qquad\qquad
      + \mathbf{D}_c^\top\mathbf{U}_c^k
      S_{\lambda\gamma}(\mathbf{\Sigma}_c^k)
      {\mathbf{V}_c^k}^\top
  \Big]
  \end{aligned}$
  \vspace{0.25em}

  \STATE $\eta_{k+1} = \frac{1 + \sqrt{1 + 4\eta_k^2}}{2}$
  \vspace{0.25em}

  \STATE $\mathbf{Y}^{k+1}
  = \mathbf{W}^k
  + \frac{\eta_k - 1}{\eta_{k+1}}
  (\mathbf{W}^k - \mathbf{W}^{k-1})$
\ENDFOR

\vspace{0.2em}
\RETURN $\mathbf{W}^k$
\end{algorithmic}
\end{algorithm}

\end{minipage}

\vspace{-0.6em}
\end{wrapfigure}

Despite this coupling, the GAME regularizer decomposes as a weighted sum of $f_c(\mathbf{W}) := \lambda\|\mathbf{W}_c\|_*$, where each $\mathbf{W}_c$ is obtained from $\mathbf{W}$ via a fixed row-subsetting operator $\mathbf{D}_c$. Specifically, let $\mathbf{W}_c=\mathbf{D}_c\mathbf{W}$, where $\mathbf{D}_c$ selects the rows belonging to category $c$. The row-selection operator is semi-orthogonal, satisfying $\mathbf{D}_c\mathbf{D}_c^\top=\mathbf{I}$. This implies that each $f_c$ has a simple proximal operator given by singular value thresholding, i.e., soft-thresholding the singular values.













This structure naturally motivates the use of the Proximal Average (PA) algorithm \citep{NIPS2013_49182f81}, which optimizes objectives composed of weighted sums of functions with simple proximal operators. PA combines the category-specific proximal updates directly, enabling efficient optimization without introducing auxiliary variables or dual updates.

We defer proofs and technical details regarding semi-orthogonality and proximal derivations of the GAME algorithm (\Cref{algorithm:game}) to \Cref{appendix:theory_for_prox_average}. With these results, we are able to obtain convergence guarantees from \citep{NIPS2013_49182f81}, highlighted in Theorem~\ref{main:yu_thm1}.

We use the abbreviation $\overline{L^2}:=\sum_{c\in\mathcal{C}}\lambda_c L_c^2=\sum_{c\in\mathcal{C}}\lambda_c r||\mathbf{D_c}||_{\text{op}}^2$ in the following theorem to denote the weighted sum of per-category Lipschitz constants.
\begin{theorem}[Theorem 1 of \citep{NIPS2013_49182f81}]\label{main:yu_thm1}
    Let $\mathbf{W}_0$ be the initialization of the PA-APG algorithm and $\mathbf{\widehat{W}}$ denote the GAME estimator from \eqref{eq:game}. Fix the accuracy $\epsilon>0$. Let step size $\gamma=\min\{1,2\epsilon/\overline{L^2}\}$. After at most $k=\sqrt{\frac{2}{\gamma\epsilon}}\|\mathbf{W}_0-\mathbf{\widehat W}\|_F^2$ steps, the PA-APG approximation, $\widetilde{\mathbf{W}}_k$, satisfies
    $$f(\widetilde{\mathbf{W}}_k)+\bar{g}(\widetilde{\mathbf{ W}}_k)\leq f(\mathbf{\widehat W})+\bar{g}(\mathbf{\widehat W})+2\epsilon.$$

\end{theorem}


\subsection{Theoretical Guarantees of GAME}

We now establish finite-sample guarantees for GAME \eqref{eq:game} under a partially observed
noisy matrix model. We observe a uniformly i.i.d.\ sample $\Omega$ of $N$
entries from $[n]\times[m]$ of the matrix $\mathbf{X}=\mathbf{W}^\star+\mathbf{E}$,
and analyze the non-spiky restricted estimator
\begin{equation}\label{eq:game-main-thm}
\widehat{\mathbf{W}}\;\in\;\arg\min_{\mathbf{W}\in\mathcal F(\alpha^\star)}
\tfrac12\bigl\|\mathcal P_\Omega(\mathbf{X}-\mathbf{W})\bigr\|_F^{2}
\;+\;\sum_{c\in\mathcal C}\lambda_c\,\|\mathbf{W}_c\|_*,
\end{equation}
where $\mathcal F(\alpha^\star):=\{\mathbf W:\|\mathbf W\|_\infty\le \alpha^\star/\sqrt{nm}\}$.

For each meta-category $c\in\mathcal C$, $\mathbf{W}_c:=\mathbf{W}[I_c,:]\in\mathbb R^{n_c\times m}$
with $n_c:=|I_c|$, $d_c:=n_c+m$, and $r_c:=\mathrm{rank}(\mathbf{W}_c^\star)$.
The per-block sample count is $N_c:=|\Omega\cap(I_c\times[m])|$. Row overlap multiplicity is
$\kappa(i):=|\{c:i\in I_c\}|$, with extremes $\kappa_{\max}$ and $\kappa_{\min}$. We now make the following assumptions.

\begin{assumption}[Subexponential Noise]\label{ass:noise-main}
The entries $E_{ij}$ are independent, mean zero, and subexponential with
parameters $(\sigma,R)$: $\mathbb E[E_{ij}^{\,p}]\le \tfrac{p!}{2}R^{p-2}\sigma^{2}$
for every integer $p\ge 2$.
\end{assumption}

\begin{assumption}[Non-Spikiness]\label{ass:spikiness-main}
There is $\alpha^\star\ge 1$ such that $\sqrt{nm}\|\mathbf{W}^\star\|_\infty
\le \alpha^\star\|\mathbf{W}^\star\|_F$, and the optimization is restricted to
$\mathcal F(\alpha^\star)$.
\end{assumption}

\begin{assumption}[Cover]\label{ass:cover-main}
$\bigcup_{c\in\mathcal C}I_c=[n]$, i.e.\ $\kappa_{\min}\ge 1$.
\end{assumption}

The spikiness constraint rules out matrices whose mass is concentrated on a
small number of entries, pathological cases not identifiable from partial
observations \citep{negahban2011restrictedstrongconvexityweighted}.

\begin{theorem}[GAME Frobenius Error]\label{thm:game-main}
Suppose Assumptions~\ref{ass:noise-main}--\ref{ass:cover-main} hold and that
$N\,n_{\min}/n\gtrsim(n+m)\log(n+m)$ with $n_{\min}:=\min_c n_c$. The choice
\begin{equation}\label{eq:lambda-main}
\lambda_c\;\gtrsim\;\frac{\sigma}{\kappa_{\min}}\sqrt{\frac{(Nn_c/n)\log d_c}{\min(n_c,m)}}
\;+\;\frac{R\log d_c}{\kappa_{\min}}
\qquad\forall\,c\in\mathcal C
\end{equation}
yields, with high probability, the restricted GAME estimator $\widehat{\mathbf{W}}$ from
\eqref{eq:game-main-thm} satisfying
\begin{equation}\label{eq:main-bd}
\frac{\|\widehat{\mathbf{W}}-\mathbf{W}^\star\|_F^{2}}{nm}
\;\lesssim\;
\frac{\kappa_{\max}^{2}}{\kappa_{\min}^{3}}\cdot\frac{\sigma^{2}\log(n+m)}{N}
\sum_{c\in\mathcal C}r_c\,(n_c\vee m),
\end{equation}
in the matrix-completion regime where $nm\log(n+m)/N\lesssim\left(\frac{R}{\sigma}\right)^2$.
\end{theorem}

The proof is given in Appendix~\ref{app:game_theory}. The bound from Theorem~\ref{thm:game-main} bounds the Frobenius error of the GAME estimator and the ground truth matrix $\mathbf{W^\star}$. It further translates into
a per-category latent-subspace guarantee via a Yu--Wang--Samworth variant of
Davis--Kahan/Wedin \citep{3fe21120-9dc1-3780-b0a4-37f58eabfcfa}, which is more interesting in the context of preserving underlying latent subspaces:

\begin{corollary}[Per-Category Right-Subspace Recovery]\label{cor:subspace-main}
For each $c\in\mathcal C$, let $\mathbf{Q}_c^\star,\widehat{\mathbf{Q}}_c\in\mathbb R^{m\times r_c}$
denote the top $r_c$ right singular vectors of $\mathbf{W}_c^\star$ and
$\widehat{\mathbf{W}}_c$, and let $\sigma_{c,1},\sigma_{c,r_c}$ be the largest
and $r_c$-th singular values of $\mathbf{W}_c^\star$. Under the hypotheses of
Theorem~\ref{thm:game-main}, with high probability, simultaneously for every
$c\in\mathcal C$,
\[
\min_{\mathbf R\in\mathrm{O}(r_c)}\bigl\|\widehat{\mathbf{Q}}_c\mathbf R-\mathbf{Q}_c^\star\bigr\|_F
\;\lesssim\;
\frac{(2\sigma_{c,1}+\|\widehat{\mathbf{W}}-\mathbf{W}^\star\|_F)\,\|\widehat{\mathbf{W}}-\mathbf{W}^\star\|_F}{\sigma_{c,r_c}^{2}}.
\]
\end{corollary}

\paragraph{Discussion.}
Theorem~\ref{thm:game-main} and Corollary~\ref{cor:subspace-main} are both
{prescriptive}, in that they convert each $\lambda_c$ and per-category
sample budget $N_c$ from free tuning parameters into quantities that can be
calibrated from observable category-level statistics, and {comparative},
in that they identify the regime where GAME strictly improves on global
nuclear-norm regularization. We analyze each in the following.

\textit{(i) Category-specific regularization.}
The choice \eqref{eq:lambda-main} scales with three intuitive
quantities: the noise level $\sigma$, the effective sample budget
$Nn_c/n$ (the expected number of observations falling in $I_c$), and the local
block dimension $d_c$. In the dominant matrix-completion regime, the
sub-Gaussian piece dominates and $\lambda_c\asymp \sigma\sqrt{(Nn_c/n)\log d_c/\min(n_c,m)}$.
Intuitively, smaller or more sparsely covered categories should receive
proportionally stronger regularization, while larger and better-sampled
categories support weaker shrinkage. This is concrete guidance that replaces
grid search with a noise-calibrated rule scaling with category-level summaries
that are themselves measurable from data.

\textit{(ii) Category-specific sample requirements.}
The sample condition $Nn_{\min}/n\gtrsim(n+m)\log(n+m)$ exposes a per-category
identifiability threshold: each category receives an effective budget
$N_c\approx Nn_c/n$, and the theorem is informative only when this budget is
large relative to $d_c$ for every relevant category, in particular for the
smallest. So GAME can exploit category structure only when each meta-category
is sufficiently sampled. If some categories are very small, practitioners
should either oversample those categories or coarsen the partition until the
threshold is met.

\textit{(iii) When GAME beats global nuclear-norm regularization.}
The dominant complexity factor in \eqref{eq:main-bd} is
$\sum_{c}r_c(n_c\vee m)$, which is a sum of \emph{local} complexities, whereas the
analogous bound for the standard nuclear-norm estimator (NW)
\citep[Cor.\,1]{negahban2011restrictedstrongconvexityweighted} carries the
\emph{global} complexity $r^\star(n\vee m)$ with $r^\star:=\mathrm{rank}(\mathbf{W}^\star)$.
Holding overlap bounded ($\kappa_{\max},\kappa_{\min}=O(1)$, i.e.\ every row
belongs to a constant number of categories), the error ratio is
\begin{equation}\label{eq:ratio-main}
\frac{\text{GAME bound}}{\text{NW bound}}
\;\asymp\;\frac{\sum_{c}r_c(n_c\vee m)}{r^\star(n\vee m)}.
\end{equation}
GAME wins whenever $\mathbf{W}^\star$ is only moderately low rank globally but
substantially lower rank within meaningful categories. As a concrete example, consider
rows partitioned into $|\mathcal{C}|$ equal blocks with common local rank $r_{\rm loc}$ and
largely distinct block subspaces (so $r^\star\asymp |\mathcal{C}|r_{\rm loc}$). In the
tall matrix regime $n_c\gg m$, \eqref{eq:ratio-main} reduces to roughly $1/|\mathcal{C}|$, so GAME
improves over NW by a factor of $|\mathcal{C}|$. By Corollary~\ref{cor:subspace-main}, the
same factor carries over to per-category subspace recovery, since the
conditioning constants $\sigma_{c,1},\sigma_{c,r_c}$ depend only on
$\mathbf{W}_c^\star$ and cancel in any ratio of subspace errors.

Further detailed discussion is deferred to Appendix~\ref{app:game_theory}. Future work may develop sharper theory under separation assumptions on the category-wise right singular subspaces, characterizing regimes in which GAME offers more robust completion while better preserving category-specific latent structure.
\section{Experiments}

The experiments investigate the impact of subgroup-aware regularization on three fronts: (1) preserving local latent structure that global methods smooth away, (2) reconstruction accuracy under both uniform and structured missingness, and (3) downstream task performance in clustering, classification, and latent dynamics recovery. We study these questions across four experiments: synthetic clustering (\Cref{exp:synthetic_clustering}), movie recommendation under global, block-wise, and corrupted missingness (\Cref{exp:movie}), species classification from imputed acoustic features (\Cref{exp:birdset}), and latent neural dynamics recovery from multi-session Neuropixels recordings (\Cref{exp:svoboda}).

Across all experiments, we compare GAME to singular value thresholding \citep[SVT;][]{doi:10.1137/080738970}, alternating least squares \citep[ALS;][]{jain2012lowrankmatrixcompletionusing}, dirty inductive matrix completion \citep[DirtyIMC;][]{NIPS2015_0609154f}, Maxide \citep{NIPS2013_e58cc5ca}, FNNM \citep{yang2022feature}, and OCMC \citep{pmlr-v267-chang25e}. We restrict DirtyIMC and Maxide to using only row-wise side information, stored in matrix $\mathbf{A}\in\mathbb{R}^{n\times d_r}$, constructed to address subspace recovery. \Cref{exp:birdset} additionally baselines against two modern matrix imputation methods: MissForest \citep{Stekhoven_2011} and TabImpute \citep{feitelberg2026tabimputeuniversalzeroshotimputation}.

Hyperparameters for baseline methods are tuned via cross-validation; for the GAME weights $\alpha_c$, the heuristic in Theorem~\ref{thm:game-main} is used as initialization, with additional scaling selected by cross-validation. Additionally, when executing the iterative proximal updates for GAME, we utilized a truncated SVD to improve runtime and memory complexity, which continued to demonstrate robust performance. For most experiments, completing the matrices with GAME is computationally feasible to run locally, although for the Svoboda Lab Neuropixel dataset we resorted to implementing a GPU compatible version of GAME for large SVD computations.

\subsection{Synthetic Clustering Under Subgroup Overlap}
\label{exp:synthetic_clustering}

\paragraph{Setup.}
We construct a synthetic matrix \(\mathbf{X}\in\mathbb{R}^{n\times m}\) with crossed group structure. Each row \(x_i\) belongs to an observed group \(g\in\mathcal{G}\) and an unobserved subcluster \(s\in\mathcal{S}\). The observed groups define the meta-categories available to GAME and the side-information baselines, while the hidden subclusters define the latent structure we aim to recover after matrix completion.

We set \(|\mathcal{G}|=10\), \(|\mathcal{S}|=5\), \(n=1000\), and \(m=500\). Each row is generated as
\[
x_{i}
=
\mu_g
+
\nu_{s}
+
u_i V_g^\top
+
(1-\beta) z_i W_s^\top,
\]
where \(\mu_g\) is a group-specific mean vector, \(\nu_s\) is a subcluster-specific mean vector, \(V_g\) is a low-rank basis associated with the observed group \(g\), and \(W_s\) is a low-rank basis associated with the hidden subcluster \(s\). The vectors \(u_i\) and \(z_i\) are normalized Gaussian latent scores that control row-level variation within the group and subcluster components, respectively.

The parameter \(\beta\in[0,1]\) controls the strength of the hidden subcluster signal. When \(\beta=0\), the subcluster-specific low-rank component is strongest; as \(\beta\to 1\), this component vanishes and the hidden subclusters become harder to recover.

\begin{figure}[t]
    \centering

    \begin{minipage}[t]{0.4\textwidth}
        \centering
        \includegraphics[
            width=\linewidth,
            trim={0.15cm 0.1cm 0.15cm 0.1cm},
            clip
        ]{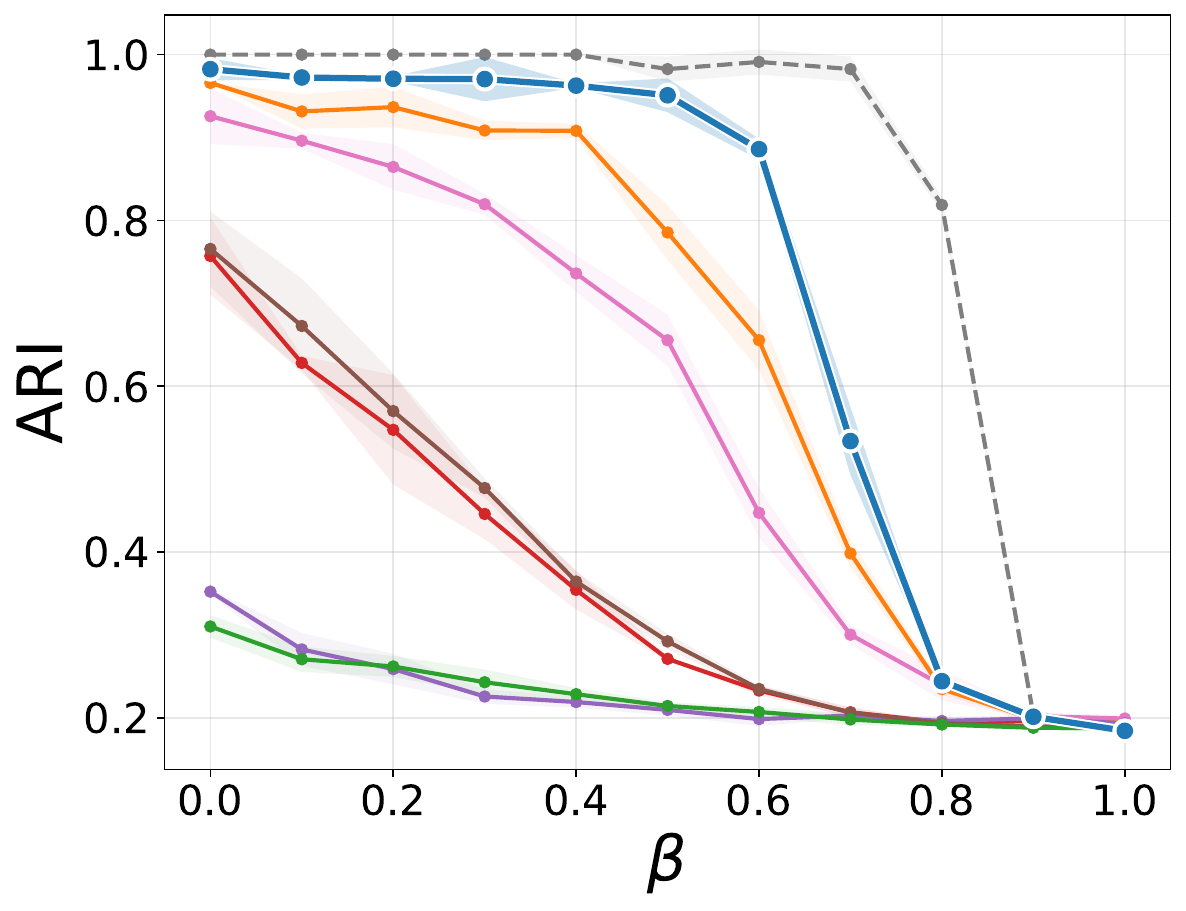}
        \vspace{0.1em}

        {\scriptsize (a) ARI}
    \end{minipage}
    \hspace{0.02\textwidth}
    \begin{minipage}[t]{0.52\textwidth}
        \centering
        \includegraphics[
            width=\linewidth,
            trim={0.15cm 0.1cm 0.15cm 0.1cm},
            clip
        ]{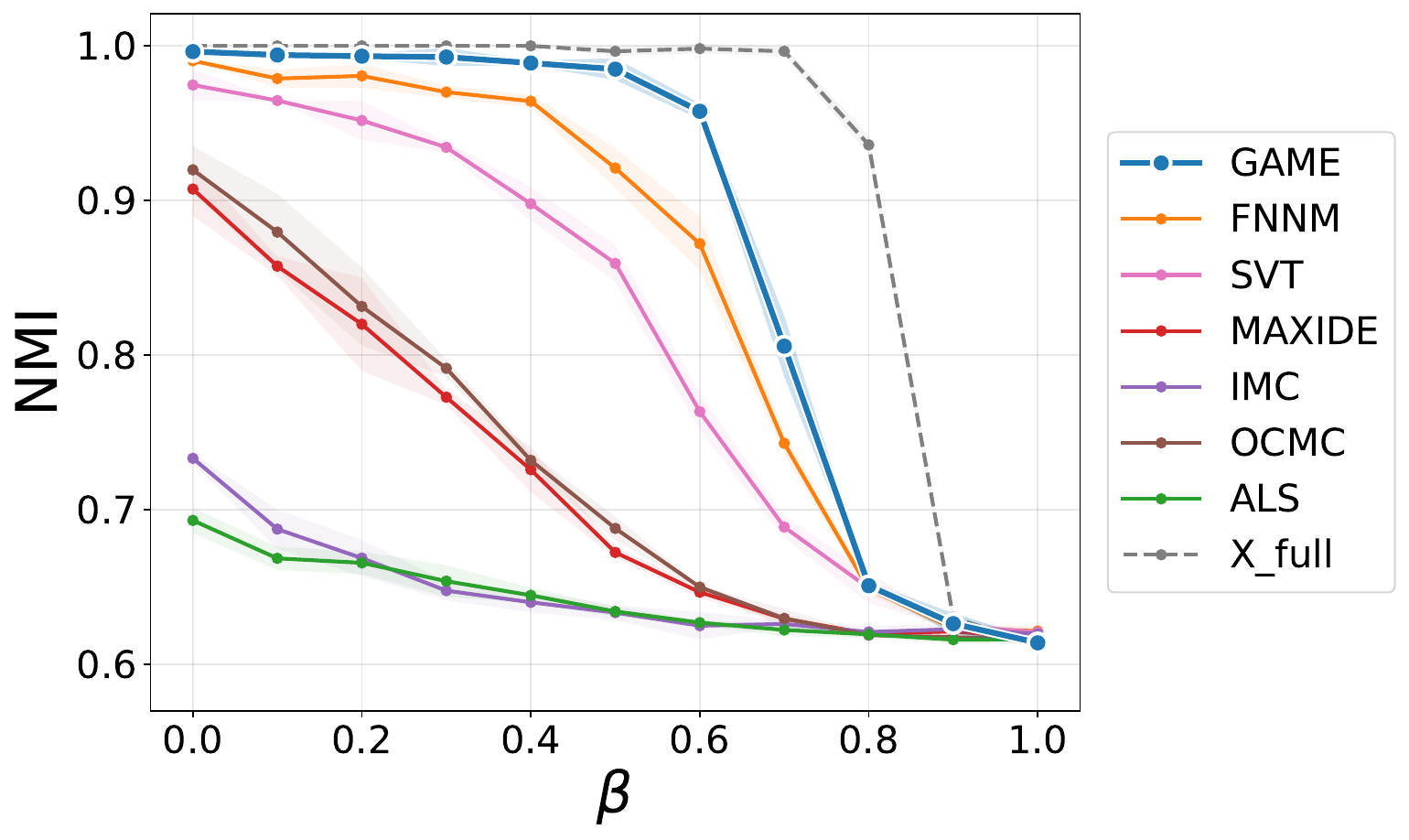}
        \vspace{0.1em}

        {\scriptsize (b) NMI}
    \end{minipage}
    \vspace{0.002cm}
    \caption{Synthetic clustering recovery of hidden subclusters under varying subcluster signal strength $\beta$, evaluated by ARI ($\uparrow$) and NMI ($\uparrow$). GAME nearly matches the fully observed oracle at small and moderate $\beta$, outperforming all baselines. Shaded regions denote $\pm 1$ standard error across replications.}
    \label{fig:clustering_beta}
\end{figure}

After generating \(\mathbf{X}\), we uniformly mask \(60\%\) of its entries and apply each matrix completion method. We then run \(K\)-means on the completed matrix and evaluate recovery of the hidden subcluster labels \(s\). This simulation can be viewed through an effects-decomposition lens: the observed group labels act as nuisance factors, and GAME uses them to regularize group-specific variation so that the completed matrix better preserves the crossed latent subcluster geometry.

\paragraph{Results.}
\Cref{fig:clustering_beta} shows that GAME most accurately recovers the hidden subclusters across small and moderate values of \(\beta\), nearly matching the fully observed oracle that clusters \(X_{\mathrm{full}}\). In contrast, global completion methods and side-information baselines degrade earlier as the subcluster signal weakens, indicating that they preserve less of the local row geometry after completion. As \(\beta\to 1\), the subcluster-specific component vanishes, the oracle performance drops, and all methods converge to similar clustering performance.

These results support the central intuition of GAME: when observed meta-categories carry low-rank structure, group-aware regularization can account for group-specific variation while preserving the crossed latent directions that define hidden subclusters.

\subsection{MovieLens-100k Dataset: Missingness and Meta-Category Robustness}
\label{exp:movie}
\begin{figure}[t]
    \centering
    \begin{minipage}[t]{0.32\linewidth}
        \centering
        \includegraphics[width=\linewidth]{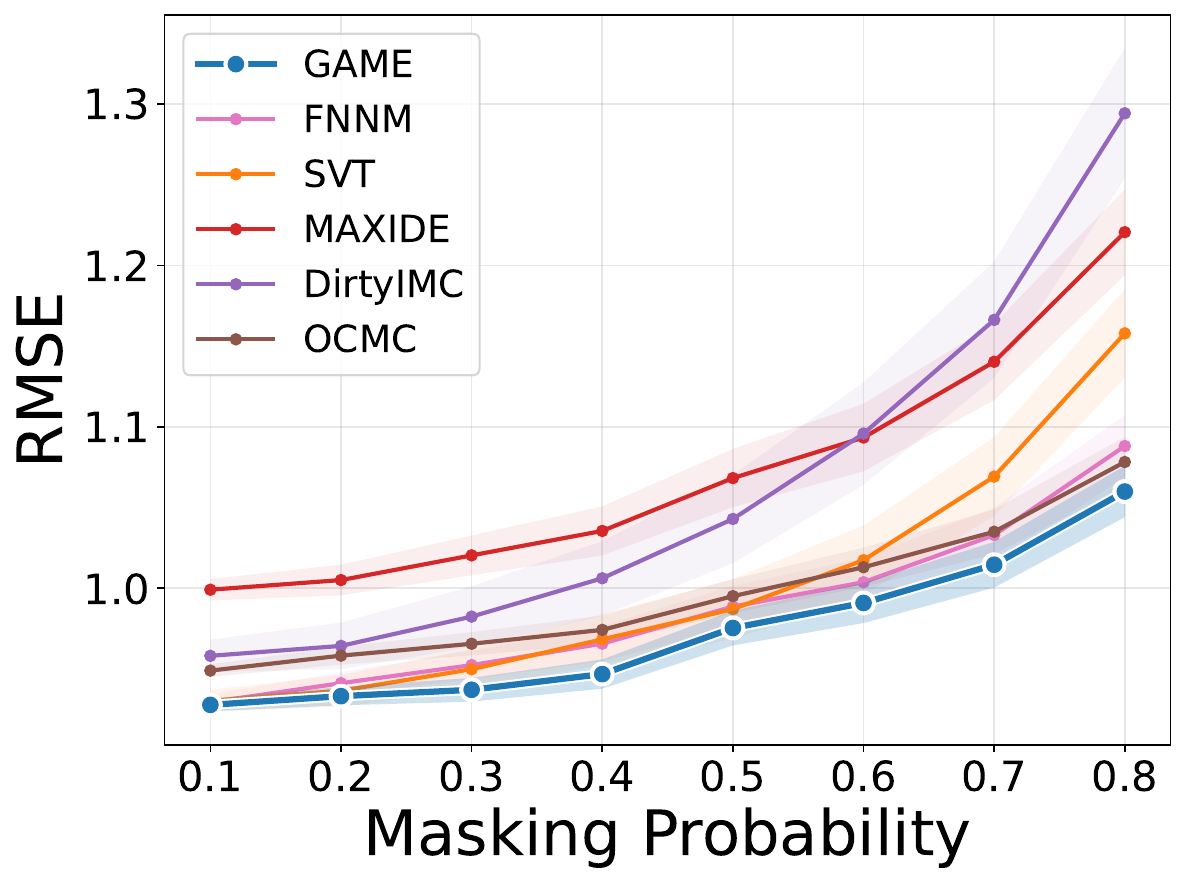}
        \vspace{0.25em}
        (a) Global missingness
    \end{minipage}
    \hfill
    \begin{minipage}[t]{0.32\linewidth}
        \centering
        \includegraphics[width=\linewidth]{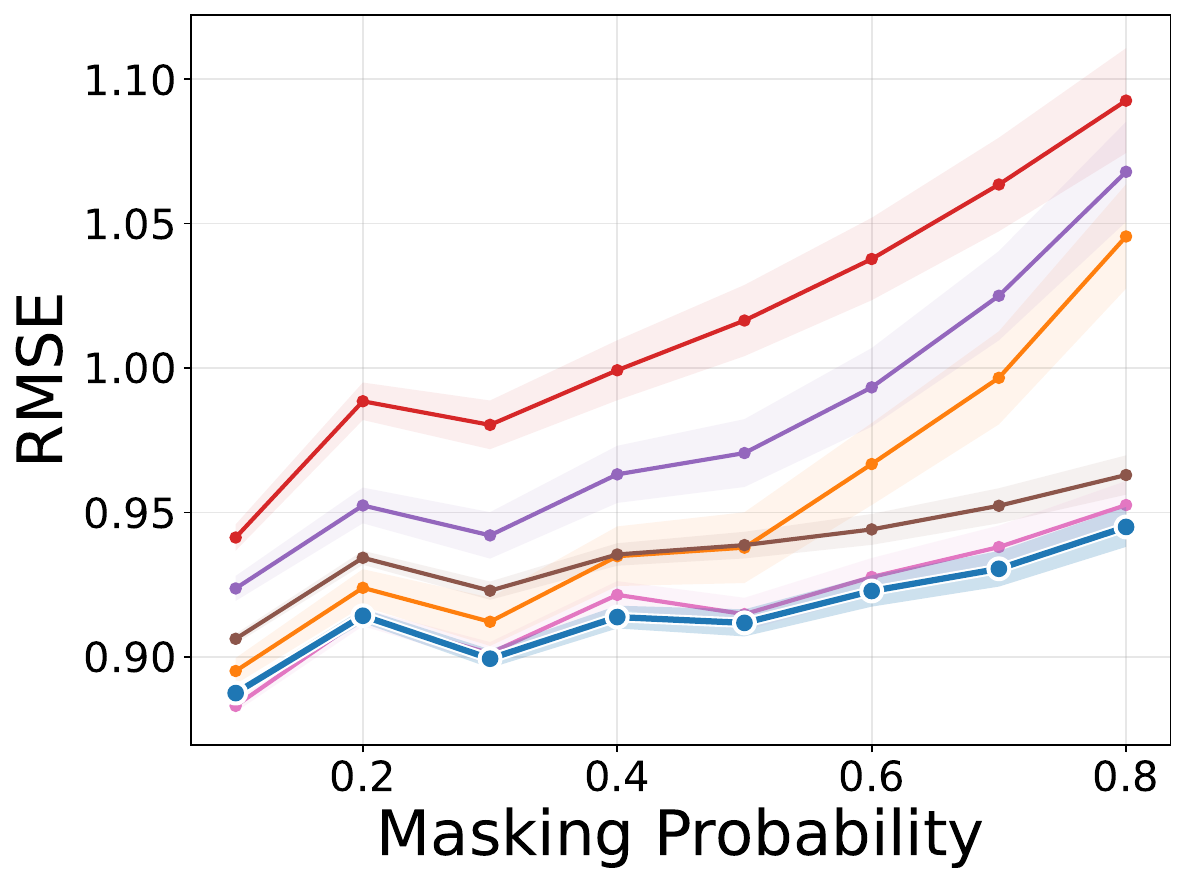}
        \vspace{0.25em}
        (b) Block-wise missingness
    \end{minipage}
    \hfill
    \begin{minipage}[t]{0.32\linewidth}
        \centering
        \includegraphics[width=\linewidth]{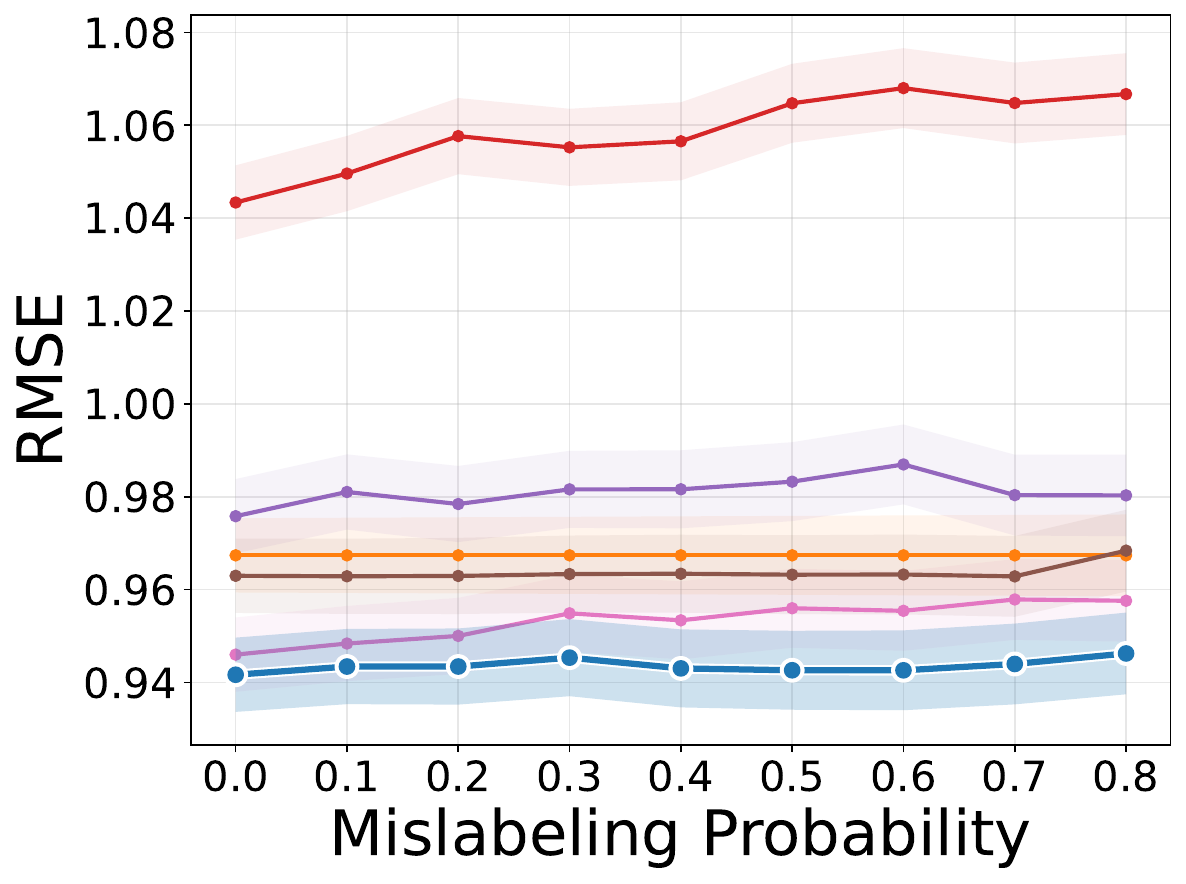}
        \vspace{0.25em}
        (c) Noisy meta-categories
    \end{minipage}

\caption{Test RMSE ($\downarrow$) on MovieLens-100k under global missingness, block-wise missingness targeting users aged $35+$, and corrupted demographic metadata. GAME achieves the lowest RMSE across all three regimes. Shaded regions denote $\pm 1$ standard error across masking realizations.}

    \label{fig:movielens_rmse}
\end{figure}

\paragraph{Setup.}
The MovieLens-100k dataset \citep{harper2015movielens} contains 100,000 ratings from 943 users on 1,682 movies, together with user metadata including gender, age, occupation, and ZIP code. We use age and gender as meta-categories and omit occupation and ZIP code. We evaluate matrix completion by holding out observed ratings and reporting test 
$$
\mathrm{RMSE}(\widehat{\mathbf W})
=
\sqrt{
\frac{1}{\lvert \Omega_{\mathrm{test}} \rvert}
\sum_{(i,j)\in \Omega_{\mathrm{test}}}
\left(
\widehat W_{ij} - X_{ij}
\right)^2
},
$$
where \(\Omega_{\mathrm{test}}\) is the held-out set of observed ratings.

We consider three regimes. First, in the \emph{global missingness} setting, we uniformly mask a proportion \(p_{\mathrm{global}}\) of observed ratings. Second, in the \emph{block-wise missingness} setting, we first hold out \(10\%\) of ratings globally, then additionally mask a proportion \(p_{\mathrm{block}}\) of ratings from users aged \(35+\). This setting tests whether methods can recover ratings for a systematically under-observed subgroup by borrowing information through overlapping demographic structure.

Third, we evaluate robustness to misspecified metadata. We impose 50\% global missingness
and corrupt the user-level meta-categories in the training set. For each user, we randomly
select either gender or age with equal probability, then corrupt the selected label with
probability $p_{\mathrm{swap}}$. We sweep $p_{\mathrm{swap}}$ from 0 to 0.8 to assess sensitivity
to unreliable demographic metadata.

\paragraph{Results.}
\Cref{fig:movielens_rmse} shows GAME is consistently among the strongest methods across the three MovieLens settings. We omit ALS because it yields substantially weaker reconstruction performance. Under global missingness, GAME achieves the lowest RMSE across all masking probabilities and remains stable as sparsity increases. Its advantage is largest under block-wise missingness, where GAME attains the lowest RMSE across nearly all values of \(p_{\mathrm{block}}\) and the gap widens as subgroup missingness increases. Under meta-category swapping, GAME remains stable as \(p_{\mathrm{swap}}\) increases and continues to outperform the side-information baselines. FNNM is the strongest competing method.

Together, these results suggest that overlapping demographic structure is most useful when missingness is itself structured by subgroup. The corruption experiment further indicates that GAME does not rely blindly on metadata: noisy labels weaken the alignment between demographic groups and rating patterns, but the estimator can still borrow strength from the observed ratings and the remaining correctly specified group memberships.

\subsection{Multi-label BirdSet: Downstream Classification}
\label{exp:birdset}
\begin{figure}[t]
    \centering
    \begin{minipage}[t]{0.285\linewidth}
        \centering
        \includegraphics[width=\linewidth]{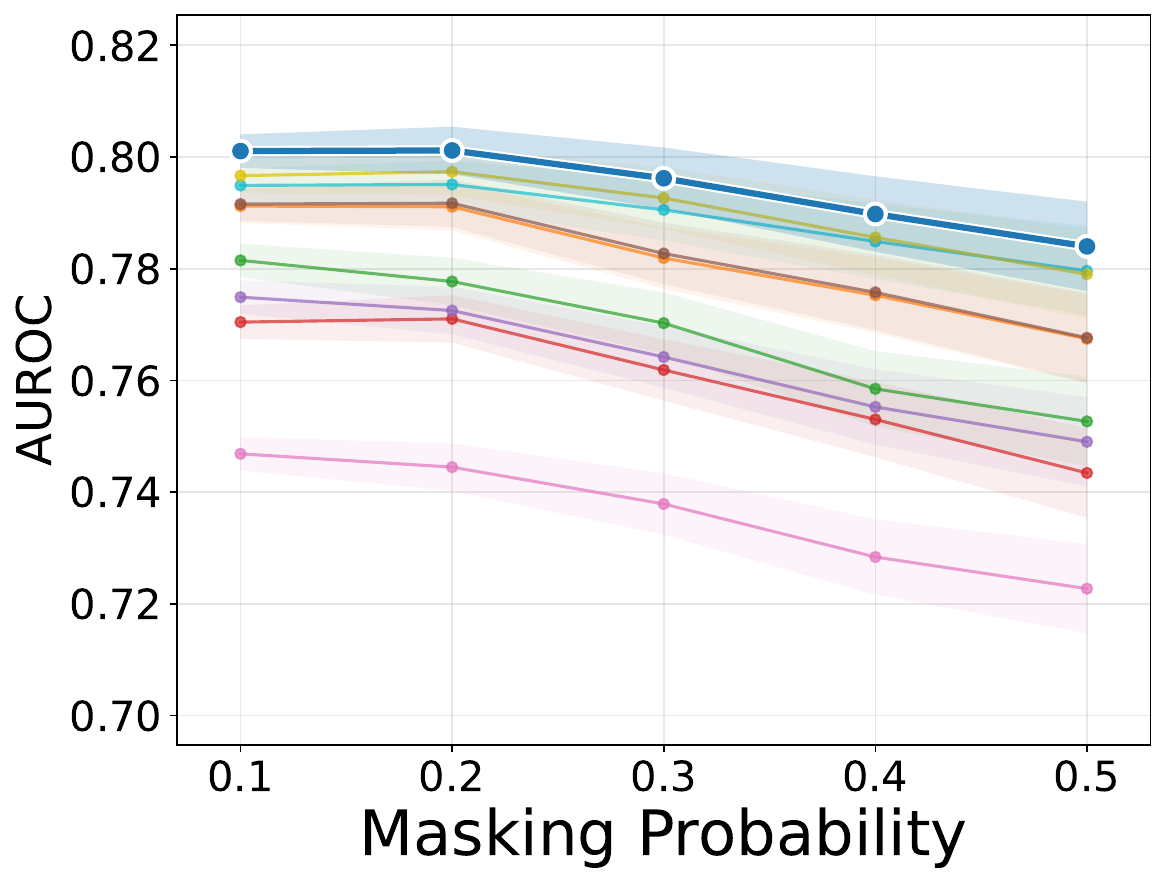}
        \vspace{0.25em}
        (a) AUROC
    \end{minipage}
    \hfill
    \begin{minipage}[t]{0.29\linewidth}
        \centering
        \includegraphics[width=\linewidth]{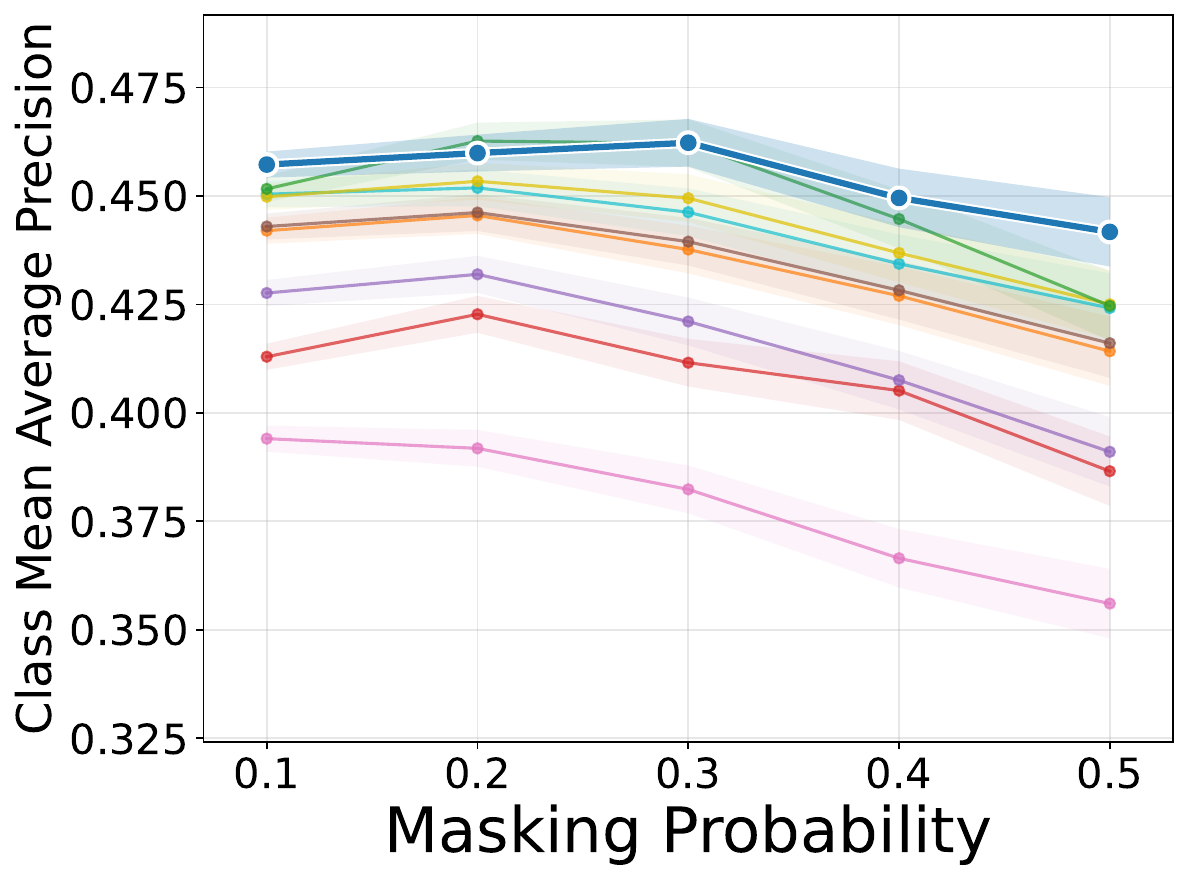}
        \vspace{0.25em}
        (b) Mean Class Precision
    \end{minipage}
    \hfill
    \begin{minipage}[t]{0.38\linewidth}
        \centering
        \includegraphics[width=\linewidth]{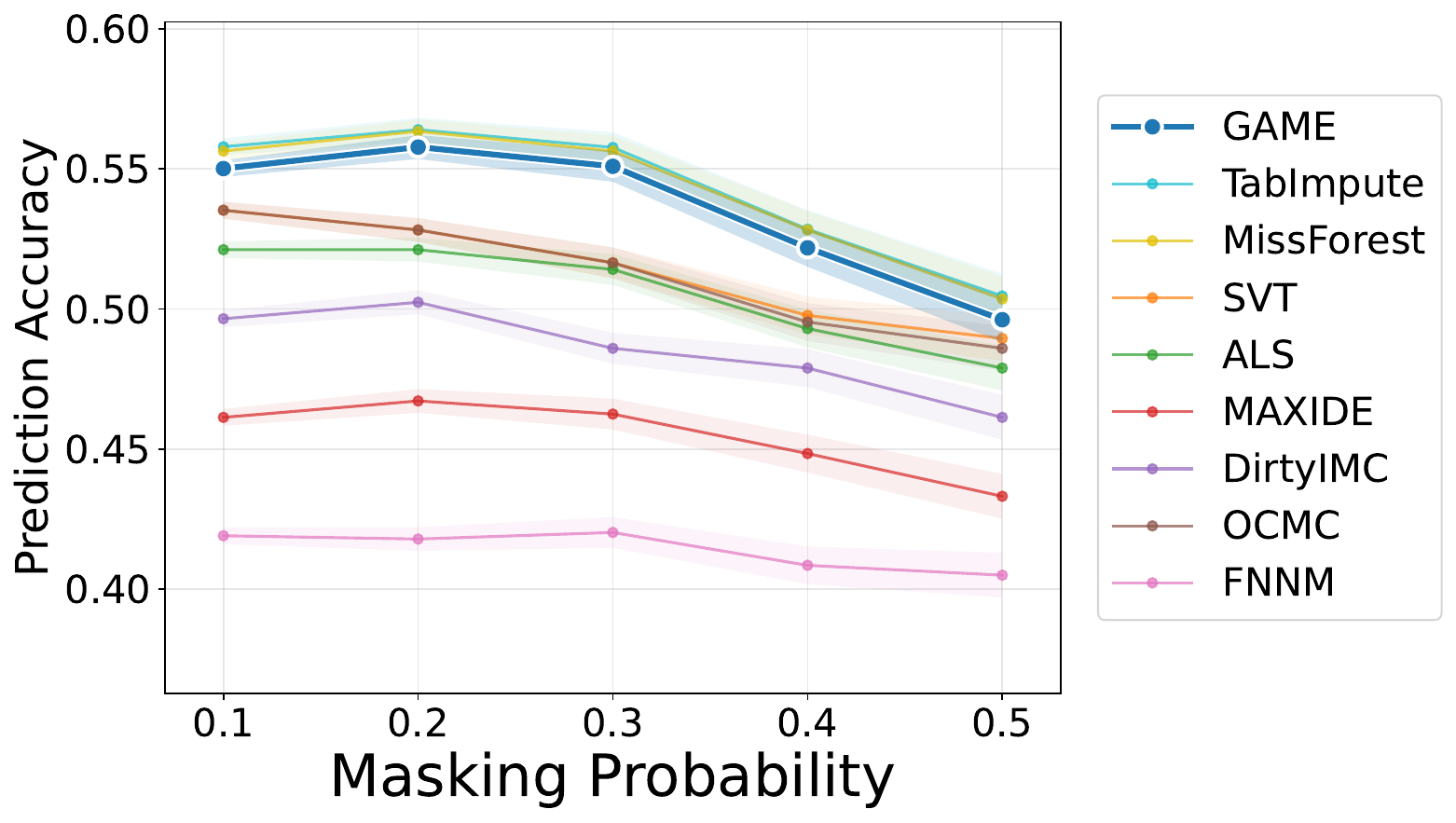}
        \vspace{0.25em}
        (c) Prediction Accuracy
    \end{minipage}

\caption{Downstream species-classification performance on BirdSet HSN after imputing masked MFCC features under label-structured missingness, evaluated by AUROC ($\uparrow$), cmAP ($\uparrow$), and top-1 prediction accuracy ($\uparrow$). GAME-completed features preserve classification signal across masking probabilities, remaining competitive or best among all methods. Shaded regions denote $\pm 1$ standard error across masking realizations.}
    \label{fig:birdsong}
\end{figure}

\paragraph{Setup.}
BirdSet \citep{rauch2025birdsetlargescaledatasetaudio} is a large-scale benchmark for birdsong classification. We use the High Sierras Nevada (HSN) subset \citep{mary_clapp_2023_7525805}, which contains 5460 focal training recordings for each of the 21 bird species, and 12,000 soundscape test recordings that may contain multiple bird species. We exclude rare species, unlabeled samples, and low-quality recordings.

Because GAME and the baselines operate on matrices, we do not apply them directly to raw audio waveforms. Instead, we convert each audio clip into a fixed-length feature vector using Mel-frequency cepstral coefficients (MFCCs), producing a feature matrix \(\mathbf{X}\in\mathbb{R}^{n\times 40}\). Rows correspond to audio clips and columns correspond to MFCC features. Here, we use bird species and recording location as meta-categorical information when applying GAME and other baselines.

We use species labels to define observed group structure for matrix completion, but the labels themselves are not imputed. We then impose structured block missingness by masking a fraction of MFCC entries among selected species. This creates a label-structured missingness pattern, corresponding to a missing-not-at-random (MNAR) setting in the imputation literature. After completing the feature matrix, we evaluate whether the imputed features preserve information useful for downstream multi-label species classification.

\paragraph{Results.}
\Cref{fig:birdsong} evaluates whether the completed MFCC features preserve information useful for downstream multi-label species classification. For each completed feature matrix, we train a random forest classifier and report AUROC, class mean average precision (cmAP), and top-1 prediction accuracy on the multi-label soundscape test set.

GAME achieves strong downstream performance across all three metrics as the masking probability increases. It is consistently among the best methods for AUROC and cmAP, indicating that the completed features preserve species-discriminative ranking information. For top-1 prediction accuracy, GAME remains competitive, although the modern tabular imputation baselines perform slightly better at some masking levels. Overall, these results suggest that group-aware completion can preserve downstream classification signal under label-structured missingness, even when the evaluation is not based directly on reconstruction error.

\subsection{Svoboda Lab Neuropixel Dataset}
\label{exp:svoboda}
\journalonly{
The Svoboda dataset \cite{CHEN2024676} consists of large-scale, simultaneous neuronal spike activity of mice performing standardized behavioral tasks (memory-guided directional licking stimulated at $t=50$), repeated across many recording sessions and spanning prespecified brain regions (Figure \ref{fig:fig1}). For a single trial, we observe a matrix $\mathbf{X} \in \mathbb{R}^{N \times T}$, where $N$=50000 is the number of neurons measured across all brain regions in the session, and $T$=400 refers to the duration of the recording. Matrix entries are $1$ if the neuron spiked at time $t$ and $0$ otherwise. \Cref{fig:session_region_heatmap} shows the structured missingness of the dataset of regions measured in 20 of the most recorded sessions. We bin the data into 10 ms windows over fixed 4-second trials, for $T=400$, and align the trials to the stimulus ($t=50$). We further restrict analyses to successful trial types from either ``hit left" or ``hit right", keeping neurons with more than 10 spikes along the 400 timepoints, and specifically select the most populous single trial type (trial ID 0). We randomly select 20 of the most measured recording sessions and 5 of the most heavily measured brain regions across those recordings. We select recording sessions 123, 146, and 134, and within these sessions, select only regions 124 (midbrain reticular nucleus), 248 (striatum), and 258 (superior colliculus). As we observe in the region-session dynamics from PCA (\Cref{fig:session_by_region_pca}) and in the mean neuron firing rates in \Cref{fig:firingrates}, the same region can have differing temporal dynamics even when controlling for task, since here we select only from a single trial type.
\begin{figure}
    \centering
    \includegraphics[width=0.7\textwidth]{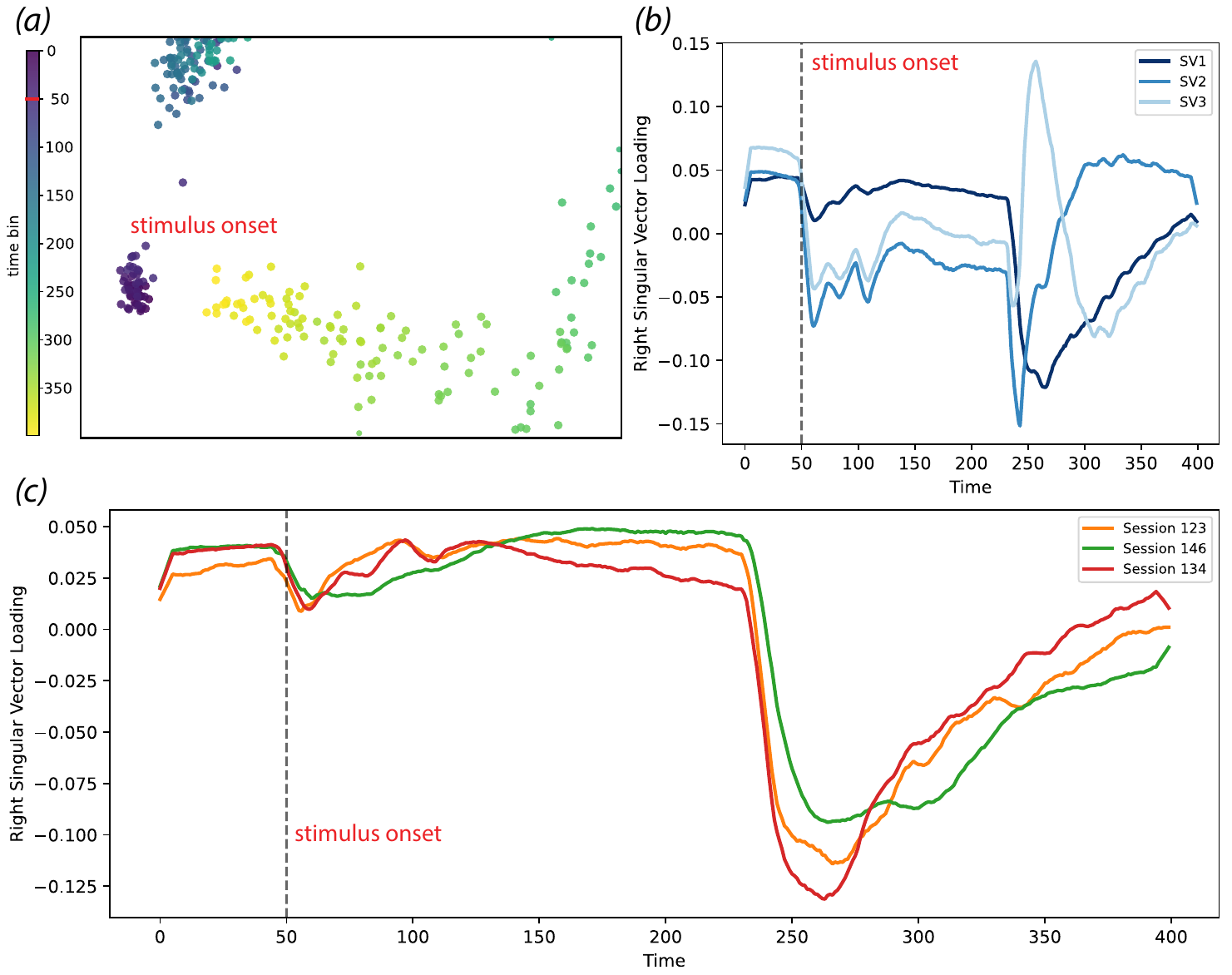}
\caption{\textbf{GAME identifies meaningful neural dynamics while accounting for time-dependent batch effects from different recording sessions.} \textit{(a)} 2D singular vectors from GAME from the striatum over time. \textit{(b)} Top 3 singular vectors show clear neural dynamic response to stimulus and from the onset of mouse movement ($\sim t=250$). \textit{(c)} Temporal session-specific batch effects from GAME.}\label{fig:248res}
\end{figure}
}
\preprintonly{
The Svoboda dataset \cite{CHEN2024676} consists of large-scale, simultaneous neuronal spike activity of mice performing standardized behavioral tasks (memory-guided directional licking), repeated across many recording sessions and spanning prespecified brain regions. Tasks, or trials, refer to the exact variant of the guided memory activity carried out in an experiment; region refers to the brain region measured, and session refers to a series of trials repeated where the Neuropixel probe is not removed, so all the same regions are measured. The Svoboda dataset consists of 173 recording sessions from 28 mice across 293 cortical and subcortical regions. Each session simultaneously records approximately 4 to 6 brain areas, each with a different number of neurons.

For a single trial, we observe a matrix $\mathbf{X} \in \mathbb{R}^{N \times T}$, where $N$ is the number of neurons measured across all brain regions in the session, and $T$ refers to the length of the task and recording. In this experiment, $N\approx 50000$ and $T=400$. Matrix entries are either $0$, if the neuron did not spike at that time, or $1$ if it did. We bin the data into 10 ms windows over fixed 4-second trials, for $T=400$, and align the trials to the stimulus ($t=50$). We further restrict analyses to successful trial types from either ``hit left" or ``hit right", keeping neurons with only more than 10 spikes along the 400 timepoints, and specifically select the most populous single trial type (trial ID 0). We randomly select 20 of the most measured recording sessions and 5 of the most heavily measured brain regions across those recordings. \Cref{fig:session_region_heatmap} shows the structured missingness of the dataset of regions measured in these specific 20 recording sessions.

We expect GAME to outperform regular matrix completion approaches if there is enough orthogonality across the meta-category row-spaces. For this reason, we also subset regions such that there is enough orthogonality in their top $k_r$ right-singular vectors. We analogously subset the sessions to have orthogonal top $k_s$ right-singular vectors. We select recording sessions 123, 146, and 134, and within these sessions, select only regions 124 (midbrain reticular nucleus), 248 (striatum), and 258 (superior colliculus). As we observe in the region-session dynamics from PCA (\Cref{fig:session_by_region_pca}), the same region can have differing temporal dynamics even when controlling for task, since here we select only from a single trial type. Across all three regions, we can visually observe heterogeneity in latent subspaces underlying memory-guided movement, further evidenced by the mean firing rates of each region across the three sessions in \Cref{fig:firingrates}. Though the single-region dynamics are similar across sessions, firing rates exhibit batch effects, and comparisons between regions are obfuscated as a result, even along only a single trial type.   

}
In this analysis, region and session latent dynamics overlap significantly, and we cannot expect $V_r\perp V_s$. Generally, region dynamics are of greater interest, where sessions are treated as nuisance parameters. We still demonstrate the effectiveness of including sessions as a meta-category in our model by comparing two GAME estimators (one using both regions and sessions, the other only regions) and comparing their recovery of latent dynamics with respect to the Grassmann distance \citep{ye2016schubert}.

Briefly, if $U,V \in \mathbb{R}^{T \times k}$ have orthonormal columns spanning $\mathcal{U}$ and $\mathcal{V}$, then the cosines of the principal angles are given by the singular values of $U^\top V$, $\cos(\theta_i) = \sigma_i(U^\top V)$ for $ i = 1,\dots,k, $ so that smaller $\theta_i$ indicate closer alignment of the two subspaces.

In particular, even though sessions are a nuisance parameter here, adding them to the GAME model can help with decontamination, since their effects may still be systematic and large, preventing region subspaces from rotating to fit session artifacts. Further, we reduce leakage, preventing sessions from corrupting regions.

\begin{wrapfigure}{r}{0.5\textwidth}
    \vspace{-2em} 
    \centering
    \includegraphics[width=0.96\linewidth]{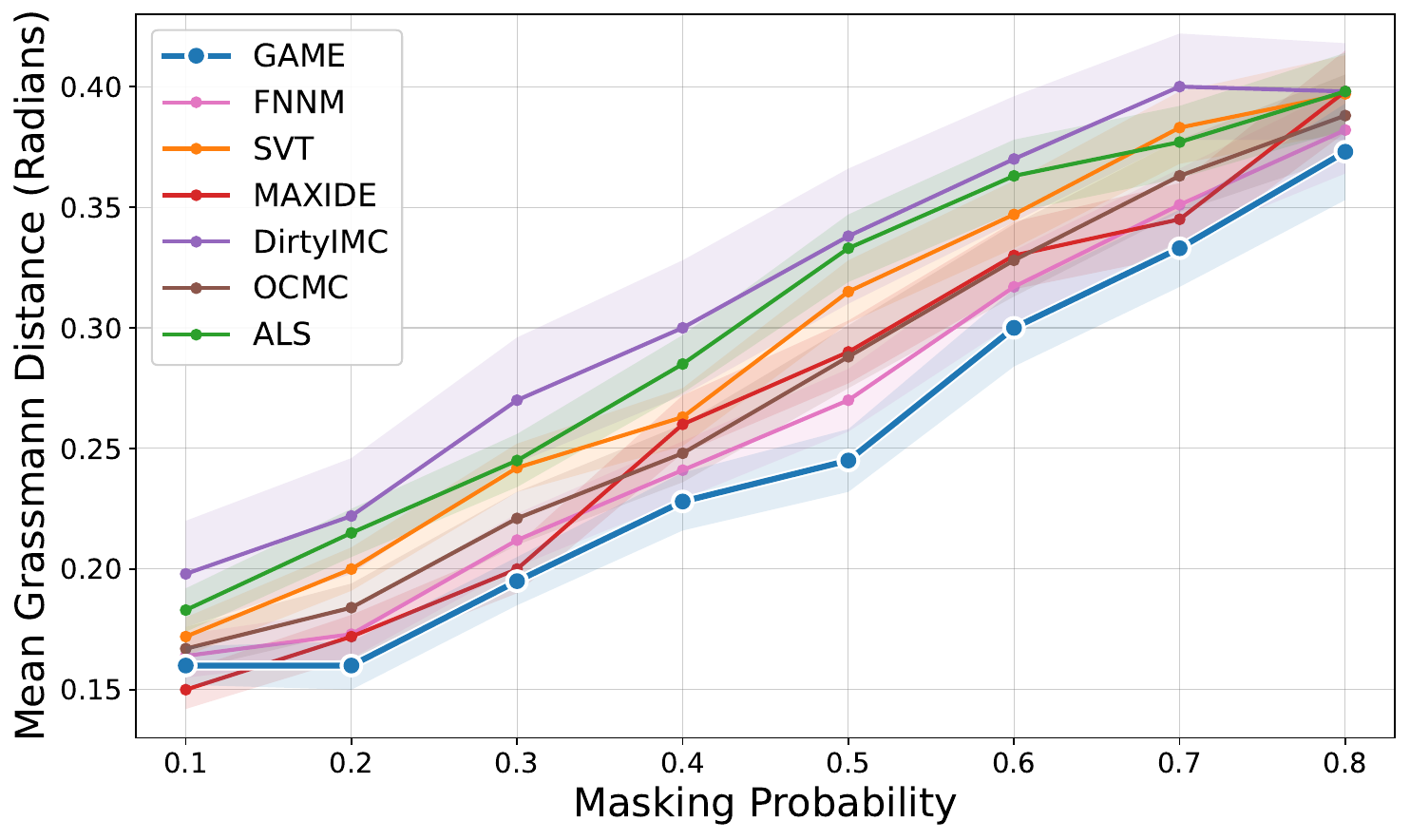}
    \caption{\textbf{Latent subspace recovery on the Svoboda Lab Neuropixels dataset},
    measured by mean Grassmann distance ($\downarrow$) between recovered and fully
    observed striatal subspaces. GAME attains the lowest distance across masking
    probabilities, indicating best recovery of latent neural dynamics. Shaded regions
    denote $\pm 1$ standard error across masking realizations.}
    \label{fig:np_grassman}
    \vspace{-1.9em}
\end{wrapfigure}

\preprintonly{
\paragraph{Results} First, we evaluate empirical results for masking on a single neural region, the striatum. We choose the striatum as it exhibits distinct temporal dynamics from other experimentally-measured regions such as the midbrain and anterior lateral motor cortex (ALM) \citep{CHEN2024676}, leading to a more difficult matrix completion problem. In \Cref{fig:np_grassman}, we evaluate GAME and alternative methods on the Grassmann distance for learning latent dynamics. We observe that GAME consistently outperforms comparing methods for learning true temporal dynamics across complex neural data. In the setting of Neuropixel data, true latent dynamics are of much greater scientific interest \citep{Pandarinath2018LFADS, churchland2012neural, vyas2020computation} than the imputation of individual neuron spikes, which are noisy and highly varying \citep{boussard2021three}.

We further investigate the performance of GAME for learning neural dynamics in the presence of batch effects in \Cref{fig:248res}. In (a), we plot the top two singular vectors over time. When compared to the striatum column (region 248) in \Cref{fig:session_region_heatmap}, and firing rates from \Cref{fig:firingrates} with differing temporal dynamics by the three selected sessions, GAME obtains a single latent dynamic, now learned across sessions while adjusting for this session-specific batch effect. Compared to the unadjusted, session-specific dynamics, we can now see a distinct separation between pre-stimulus neural dynamics ($0 < t  < 50)$, then pre-movement neural dynamics, and lastly movement onset $(\sim250 < t<400)$. The top 3 singular vectors are further visualized in (b), where we can see a specific response to stimulus onset ($t=50$), and additional dynamics near the onset of physical movement in the mouse at around $t=250$. In (c), we visualize the top singular values per session effect for the three neural recording sessions. We observe minimal session-specific batch effects through stimulus onset, with significant batch effects emerging during the beginning of animal movement. Although in the original analysis, \citep{CHEN2024676} quantifies dynamics in the striatum and other regions by grouping all recording sessions together, the authors annotate region-specific dynamics with extensive domain expertise and experimental validation. Here, our results for learned striatum dynamics recapitulate the response of that region in \citep{CHEN2024676}, but without costly electrophysiological measurements or significant neuroscientific knowledge.
}
\preprintonly{
\paragraph{Results} We evaluate empirical results for masking the striatum, as it exhibits distinct temporal dynamics from other experimentally-measured regions \citep{CHEN2024676}, leading to a more difficult matrix completion problem. In \Cref{fig:np_grassman}, we evaluate GAME and alternative methods on Grassmann distance for learning latent dynamics. We observe that GAME consistently outperforms comparing methods for learning true temporal dynamics across complex neural data. In the setting of Neuropixel data, true latent dynamics are of much greater scientific interest than the imputation of individual neuron spike trains \citep{Pandarinath2018LFADS, churchland2012neural, vyas2020computation}, which can be noisy and highly varying \citep{boussard2021three}. 

In \Cref{fig:248res} (a), we plot the top two singular vectors over time. When compared to the striatum column (region 248) in \Cref{fig:session_region_heatmap}, and firing rates from \Cref{fig:firingrates}, GAME obtains a single latent dynamic, now learned across sessions while adjusting for this session-specific batch effect. We can now see a distinct separation between pre-stimulus neural dynamics ($0 < t  < 50)$, then pre-movement neural dynamics, and lastly movement onset $(\sim250 < t<400)$. The top 3 singular vectors are further visualized in (b), where we can see a specific response to stimulus onset ($t=50$), and additional dynamics near the onset of physical movement in the mouse at around $t=250$. In (c), we visualize the top singular values per session effect for the three neural recording sessions. We observe minimal session-specific batch effects through stimulus onset, with significant batch effects emerging at the beginning of animal movement. Although in the original analysis, \citep{CHEN2024676} quantifies dynamics in the striatum and other regions by grouping all recording sessions together, the authors annotate region-specific dynamics with extensive domain expertise and experimental validation. Here, our results for learned striatum dynamics recapitulate the response of that region in \citep{CHEN2024676}, but without costly electrophysiological validation or significant neuroscientific knowledge.
}

\journalonly{

\paragraph{Results} We evaluate matrix completion when masking the striatum, a region with distinct temporal dynamics from other experimentally measured regions \citep{CHEN2024676}. As shown in \Cref{fig:np_grassman}, GAME consistently outperforms comparison methods in Grassmann distance, indicating more accurate recovery of latent neural dynamics. In the setting of Neuropixel data, true latent dynamics are of much greater scientific interest \citep{Pandarinath2018LFADS, churchland2012neural, vyas2020computation} than the imputation of individual neuron spike trains which are noisy and highly varying \citep{boussard2021three}.

In \Cref{fig:248res}(a), the top two singular vectors show a shared latent dynamic learned across sessions after adjusting for session-specific effects; compared with the striatum column in \Cref{fig:session_region_heatmap} and firing rates in \Cref{fig:firingrates}, these dynamics separate pre-stimulus activity ($0<t<50$), pre-movement activity, and movement onset ($\sim250<t<400$). \Cref{fig:248res}(b) further shows the top three singular vectors, revealing stimulus-locked activity near $t=50$ and additional movement-related dynamics near $t=250$, while \Cref{fig:248res}(c) shows that session-specific effects are minimal through stimulus onset but become more pronounced during movement. Together, these results recapitulate the striatal dynamics identified in \citep{CHEN2024676}, but do so directly from the learned decomposition without requiring extensive electrophysiological validation or domain-specific annotation.

}
\vspace{-0.5em}\section{Discussion}

In this work, we presented Group-Aware Matrix Estimation (GAME), a convex framework for heterogeneous matrix completion that leverages overlapping meta-categorical information to promote subgroup-wise low-rank structure. We developed finite-sample theoretical guarantees for both Frobenius reconstruction error and per-category subspace recovery, showing that GAME's statistical complexity depends on the sum of local category ranks rather than the global rank, yielding strict improvements over standard nuclear-norm regularization when subgroups exhibit distinct low-rank geometry in the tall matrix regime. We further introduced a scalable proximal-averaging optimization procedure that avoids the consensus overhead of ADMM, making GAME practical even when the number of overlapping categories is large.

The breadth of our experimental evaluation underscores GAME's generalizability as a tool across disciplines, tasks, and data modalities. From collaborative filtering on user-movie ratings, to multi-label classification from acoustic features in ecological monitoring, to recovery of latent neural dynamics from large-scale electrophysiological recordings, GAME consistently performed competitively or best, with the largest gains appearing when subgroups carried distinct low-rank structure, and missingness was non-uniform. These results suggest that GAME is not narrowly suited to any single application domain but rather serves as a broadly applicable framework for matrix completion and latent subspace recovery wherever heterogeneous observations admit overlapping categorical structure. GAME's ability to simultaneously reconstruct missing entries and preserve local latent geometry positions it as a practical tool for scientific and operational settings in which downstream analyses depend on faithful recovery of subgroup-specific variation.

A current limitation is that GAME's advantages depend on the meta-categories capturing meaningful subgroup-specific low-rank structure; when metadata are weakly informative or unrelated to the missingness pattern, the estimator may offer limited improvement over global methods. The present theoretical results also rely on block-wise uniform sampling, and extending these guarantees to general sampling conditions using techniques from \citep{Klopp_2014} remains an important open direction. On the experimental side, future work on Neuropixels data will incorporate meta-categories for different trials and different mice, further testing GAME's capacity to disentangle overlapping sources of neural variability. The Svoboda data in full contains 293 neural regions, 173 recording sessions, and can be further expanded to include meta-categories for different trials and different mice \citep{CHEN2024676}, in addition to big-data extensions to incorporate other publicly available Neuropixels data from the International Brain Laboratory \citep{abbott2017international} and include site as a meta-category.

More broadly, the principle that regularization should reflect known categorical structure in the data, rather than imposing a single monolithic low-rank assumption, extends naturally beyond the settings studied here. Domains such as multi-omics integration \citep{subramanian2020multi}, where samples span overlapping tissue types, disease subtypes, and assay platforms with block-structured missingness, and electronic health records (EHR), where clinical observations are indexed by overlapping demographic and diagnostic categories partially observed depending on site \citep{wells2013strategies, teoh2024advancing}, present immediate opportunities for group-aware estimation. By unifying scalable optimization, finite-sample theory, and empirical validation across diverse domains, GAME establishes group-aware regularization as a principled and practical paradigm for high-dimensional heterogeneous matrix estimation.

\section*{Acknowledgments}
G.I.A acknowledges funding from NSF DMS-1554821. We thank Ji Xia for her help and expertise in processing the Svoboda Neuropixels data. 
\bibliographystyle{plainnat}
\bibliography{Preprint/refs}

\newpage
\appendix
\crefalias{section}{appendix}
\crefalias{subsection}{appendix} 
\section{Theory for Proximal Averaging Convergence Guarantees}
\label{appendix:theory_for_prox_average}

\subsection{Proximal Averaging Background}\label{app:prox_avg_backg}

\begin{definition}[Proximal Operator] Let $x,z\in\mathbb{X}$, where $\mathbb{X}$ is a real Hilbert space. We define the \textit{proximal operator} of a function $f$ as a mapping
\begin{equation*}
    \text{prox}_{\gamma f}(x)=\arg\min_z \dfrac{1}{2\gamma}||x-z||_2^2 + f(z).
\end{equation*}
\end{definition}

A function is \textit{simple} if it has a closed form proximal operator. Proximal operators exhibit many desirable algebraic properties that prove to be extremely useful in solving convex optimization problems. A particularly nice property of proximal operators is utilized in constructing the GAME solution. More generally, an optimization procedure known as \textit{proximal gradient descent} is successful if the objective is separable into a convex and differentiable component $f$, and another convex, but not necessarily differentiable, component $h$ \citep{OPT-003}:
$$\arg\min_{x\in\mathbb{X}} f(x)+h(x).$$

To minimize the following general multi-term, nonsmooth convex minimization problem

\begin{equation}\label{prox_avg}
\arg\min_{x\in\mathbb{X}}F(x)=f(x)+\bar{g}(x), \quad \text{where } \bar{g}(x)=\sum_{k=1}^K \alpha_k g_k(x)
\end{equation}
Here, $\alpha_k\geq 0$ satisfying $\sum_{k=1}^K \alpha_k = 1$, $g_k:\mathbb{X}\rightarrow [-\infty,\infty]$ is a proper closed convex function, and $f:\mathbb{X}\rightarrow (-\infty,\infty)$ is a continuously differentiable and gradient $L_f$ Lipschitz convex function. This is exactly the minimization problem \citep{NIPS2013_49182f81} considered, relying on proximal averaging approximations to solve the \eqref{prox_avg}. 
\\

\begin{definition}[Proximal Average, Bauschke et al. 2008 \citep{doi:10.1137/070687542}] We define the \textit{Proximal Average function} The Proximal Average function of $g(x)=\sum_{k=1}^K \alpha_k g_k(x)$ with parameter $\gamma>0$ is defined by the optimization problem
$$g_\gamma(x):=\min_{y_k}\left\{
\sum_{k=1}^K \alpha_k g_k(y_k) + \frac{1}{2\gamma} \left( \sum_{k=1}^K \alpha_k \|y_k\|^2 - \|x\|^2 \right)
\quad : \quad \sum_{k=1}^K \alpha_k y_k = x
\right\}$$
\end{definition}

The proximal operator of the Proximal Average function satisfies the nice property
$$\text{prox}_{g_\gamma,\gamma}=\alpha_1\text{prox}_{g_1,\gamma}+\cdots+\alpha_N\text{prox}_{g_N,\gamma}$$
This result says that the proximal operator of $g_\gamma$ with respect to the same step size $\gamma$ is additive in the individual proximal operators of each $g_k$. The power of this formulation stems from the fact that the proximal operator of the function $\bar{g}(x)$ might not have a nice closed form equation, but each of the individual $g_k$'s do, allowing us to construct an approximation with convergence guarantees. With the construction of a Proximal Average, \citep{NIPS2013_49182f81} solve \eqref{prox_avg} via the approximation
\begin{equation}
    \arg\min_{x\in\mathbb{X}} f(x)+g_\gamma(x)
\end{equation}

Observe that this problem is of the same form as proximal gradient descent. \citep{NIPS2013_49182f81} makes this exact connection, plugging the approximate objective function into the updates provided in proximal gradient descent. \citep{NIPS2013_49182f81} also provides theoretical justifications for the aforementioned algorithm, under a few assumptions. 

In particular, their proposed algorithm takes the form of the Fast Iterative Shrinkage–Thresholding Algorithm (FISTA) \citep{doi:10.1137/080716542} outlined in Algorithm \ref{alg:PA-APG}, as well as a non-accelerated ISTA in Algorithm \ref{alg:PA-PG}. The accelerated version is referred to as ``accelerated proximal gradient" (PA-APG), whereas the non-accelerated counterpart is PA-PG.

\begin{figure}[t]
\centering
\begin{minipage}[t]{0.47\linewidth}
\begin{algorithm}[H]
\caption{PA-APG \citep{NIPS2013_49182f81}.}
\begin{spacing}{1.3}
\begin{algorithmic}[1]\label{alg:PA-APG}
\STATE Initialize $x_0 = y_1$, $\gamma$, $\eta_1 = 1$.
\FOR{$t = 1,2,\dots$}
    \STATE $z_t = y_t - \gamma \nabla \ell(y_t)$,
    \STATE $x_t = \sum_k \alpha_k \cdot \text{prox}_{\gamma f_k}(z_t)$,
    \STATE $\eta_{t+1} = \dfrac{1 + \sqrt{1 + 4 \eta_t^2}}{2}$,
    \STATE $y_{t+1} = x_t + \dfrac{\eta_t - 1}{\eta_{t+1}} (x_t - x_{t-1})$.
\ENDFOR
\end{algorithmic}
\end{spacing}
\end{algorithm}
\end{minipage}
\hfill
\begin{minipage}[t]{0.47\linewidth}
\begin{algorithm}[H]
\caption{PA-PG \citep{NIPS2013_49182f81}.}
\begin{spacing}{2.12} 
\begin{algorithmic}[1]\label{alg:PA-PG}
\STATE Initialize $x_0$, $\gamma$.
\FOR{$t = 1,2,\dots$}
    \STATE $z_t = x_{t-1} - \gamma \nabla \ell(x_{t-1})$,
    \STATE $x_t = \sum_k \alpha_k \cdot \text{prox}_{\gamma f_k}(z_t)$.
\ENDFOR
\end{algorithmic}
\end{spacing}
\end{algorithm}
\end{minipage}
\end{figure}

Denote $\bar{L}^2=\sum_{i=k}^K \alpha_k L_k^2$ as a convex combination of each Lipschitz constant $L_k$.

\begin{theorem}[Theorem 1 of \citep{NIPS2013_49182f81}]\label{thm1}
    Assuming the following:
    
    \begin{enumerate}
        \item Each $g_i$ is convex and $L_i$-Lipschitz continuous w.r.t. the Euclidean norm.
        \item Each $g_i$ is simple.
        \item $f$ is convex with $L_f$-Lipschitz contniuous gradient w.r.t. the Euclidean norm.
    \end{enumerate}

    Let $x_0$ be the initialization of the PA-APG algorithm. Fix the accuracy $\epsilon>0$. Under assumptions 1-3, let $\gamma=\min\{1/L_f,2\epsilon/\overline{L^2}\}$. After at most $\sqrt{\frac{2}{\gamma\epsilon}}\|x_0-x\|_2$ steps, the PA-APG approximation, $\tilde{x}$, satisfies
    $$f(\tilde{{x}})+\bar{g}(\tilde{{x}})\leq f(x)+\bar{g}(x)+2\epsilon.$$

    The same guarantee holds for the non-accelerated Proximal Average approximation after at most $\dfrac{1}{2\gamma\epsilon}\|\mathbf{W}_0-\mathbf{W}\|_F^2$ steps.
\end{theorem}

\subsection{Solving GAME via Proximal Averaging}

The Proximal Average algorithm \citep{NIPS2013_49182f81} requires $\sum_{c \in \mathcal{C}} \alpha_c = 1$. This yields us our final GAME optimization problem
\begin{equation}\label{appendix:opt_game}
    \arg\min_{\mathbf{W}} \frac{1}{2}\|\mathbf{X}-\mathbf{W}\|^2_F + \lambda\sum_{c \in \mathcal{C}} \alpha_c \|\mathbf{D}_c\mathbf{W}\|_*,
\end{equation}
where we define $\lambda_c := \alpha_c \lambda$ to represent proportions of regularization for each meta-category $c\in\mathcal{C}$. As discussed in the previous subsection, to leverage proximal averaging techniques for solving GAME \eqref{appendix:opt_game}, we first require $f(\mathbf{W}):=||\mathbf{X}-\mathbf{W}||_F^2$ to be convex with $L_f$-Lipschitz continuous gradient (with respect to the Euclidean norm), and each $g_c(\mathbf{W}):=||\mathbf{D_cW}||_*$ to be convex and $L_c$-Lipschitz continuous with repsect to the Euclidean norm. Secondly, we require that each $g_c$ is simple (i.e. has a defined closed-form proximal operator). We claim that the assumptions hold for our GAME formulation \eqref{appendix:opt_game} through the following propositions, constructing an analogous theorem to \ref{thm1}.

It should be noted that although \citep{NIPS2013_49182f81} defines their results for $x\in\mathbb{R}^d$ equipped with the $l_2$ norm, but this is not an issue since our results extend to $\mathbf{X}\in\mathbb{R}^{n\times m}$ equipped with the Frobenius norm without loss of generality. This is because the $l_2$ norm and Frobenius norm coincide under vectorization, i.e. $\|\text{vec}(\mathbf{X})\|_2=||\mathbf{X}||_F$, meaning the two metric spaces are isometric. Further, all technical tools used by \citep{NIPS2013_49182f81} pertain to a real Hilbert-space domain. Hence, all Lipschitz, convexity, and proximal arguments transfer verbatim.

We recognize the GAME formulation \eqref{appendix:opt_game} as a Proximal Average problem \citep{NIPS2013_49182f81}, where $g_i$'s are simple (easy to compute proximal operator). Thus, we aim to massage the objective function into a composition of simple functions, but this formulation appears to be complex due to the penalties that depend on slices of $\mathbf{W}$ that are not disjoint (i.e. there can exist $c,c'\in\mathcal{C}$ such that $c\cap c' \neq \emptyset$). Recall that $\mathbf{W}_c$ is the row-wise ``slice" of $\mathbf{W}$ that corresponds to the $c$ meta-category block.

The key observation to make at this stage is that $\mathbf{D_cD_c}^\top = \mathbf{I}$, i.e. $\mathbf{D_c}$ is semi-orthogonal. Using Table 10.1 from \citep{combettes2010proximalsplittingmethodssignal}, we may rewrite the proximal operator of $g_c$ in terms of linear operations of $\mathbf{D_c}$ and the proximal operator of the nuclear-norm, given in Proposition \ref{appendix:semi-orth}.

\begin{proposition}[Table 10.1 of \citep{combettes2010proximalsplittingmethodssignal}]\label{appendix:semi-orth} 
Define $g_c(\mathbf{W}):=||\mathbf{D_cW}||_*$. Then,
    $$\text{prox}_{g_c} (\mathbf{A}) = \text{prox}_{||\mathbf{D_cA}||_*} (\mathbf{A}) =  \mathbf{A} - \mathbf{D}_c^\top (\mathbf{D}_c\mathbf{A} - \text{prox}_{||\cdot||_*}(\mathbf{D}_c\mathbf{A})) $$
\end{proposition}

From Proposition \ref{appendix:semi-orth}, we see the nuclear-norm having a simple proximal operator is sufficient for $g_c$ to have a simple proximal operator. The closed form proximal operator of the nuclear-norm is a well established fact \citep{OPT-003}. 

\begin{proposition}[nuclear-norm Proximal Operator \citep{OPT-003} Section 6.7.3]\label{appendix:nuc_prox}
    Let $\lambda > 0$ and let $\mathbf{Y} \in \mathbb{R}^{m \times n}$ have the singular value decomposition of rank $r$ 
    \[\mathbf{Y}
        = \mathbf{U}\,\boldsymbol{\Sigma}\,\mathbf{V}^\top,\;
        \boldsymbol{\Sigma}
        = \operatorname{diag}(\sigma_1,\dots,\sigma_r),\;
        \sigma_1\ge \cdots \ge \sigma_r > 0.\]
    Then the proximal operator of the nuclear-norm,
    \[
        \mathrm{prox}_{\lambda \|\cdot\|_*}(\mathbf{Y})
        \;:=\;
        \argmin_{\mathbf{X} \in \mathbb{R}^{m \times n}}
        \Big\{
        \tfrac{1}{2}\|\mathbf{X}-\mathbf{Y}\|_F^2
        + \lambda \|\mathbf{X}\|_*
        \Big\},
    \]
    is given by singular value soft-thresholding:
    \[
        \mathrm{prox}_{\lambda \|\cdot\|_*}(\mathbf{Y})
        \;=\;
        \mathbf{U}\,\operatorname{diag}((\sigma_i - \lambda)_+) \mathbf{V}^\top,
    \]
    where $(t)_+ := \max\{t,0\}$ is the positive part applied entrywise to the singular values.
\end{proposition}

The following series of propositions demonstrate that assumptions made by \citep{NIPS2013_49182f81} (found in Theorem \ref{thm1} to establish convergence guarantees) are met by the GAME estimator. 
\begin{proposition}\label{appendix:lipschitzness_loss}
    Let $\mathbf{X}\in\mathbb{R}^{n\times m}$ be fixed and let 
    $\mathbf{M}\in\{0,1\}^{n\times m}$ be a binary mask encoding the observed
    entries, i.e.\ $\mathbf{M}_{ij}=1$ if $(i,j)$ is observed and 
    $\mathbf{M}_{ij}=0$ otherwise. Define the orthogonal projection
    $\mathcal{P}_{\Omega}:\mathbb{R}^{n\times m}\to\mathbb{R}^{n\times m}$ by
    \[
        \mathcal{P}_{\Omega}(\mathbf{Z})
        \;:=\;
        \mathbf{M}\odot \mathbf{Z},
    \]
    where $\odot$ denotes the Hadamard (entrywise) product. Consider the function
    \[
        f(\mathbf{W})
        \;:=\;
        \frac{1}{2}\bigl\|\mathcal{P}_{\Omega}(\mathbf{X}-\mathbf{W})\bigr\|_F^2,
        \qquad \mathbf{W}\in\mathbb{R}^{n\times m}.
    \]
    Then $f$ is convex and differentiable, with gradient
    \[
        \nabla f(\mathbf{W})
        \;=\;
        \mathcal{P}_{\Omega}(\mathbf{W}-\mathbf{X})
        \;=\;
        \mathbf{M}\odot (\mathbf{W}-\mathbf{X}).
    \]
    Moreover, $\nabla f$ is $L_f$–Lipschitz continuous with respect to the Frobenius norm with constant $L_f = 1,$
    that is, for all $\mathbf{W}_1,\mathbf{W}_2\in\mathbb{R}^{n\times m}$,
    \begin{align*}
        \bigl\|\nabla f(\mathbf{W}_1)-\nabla f(\mathbf{W}_2)\bigr\|_F
        &\;\le\;
        L_f\,\bigl\|\mathbf{W}_1-\mathbf{W}_2\bigr\|_F \\
        &\;=\;
        \bigl\|\mathbf{W}_1-\mathbf{W}_2\bigr\|_F.
    \end{align*}
\end{proposition}

\begin{proof}
We first rewrite $f$ using the definition of $\mathcal P_\Omega$:
\[
  f(\mathbf W)
  = \frac12 \bigl\|\mathcal P_\Omega(\mathbf X - \mathbf W)\bigr\|_F^2
  = \frac12 \sum_{i,j} \bigl(\mathbf M_{ij}(\mathbf X_{ij} - \mathbf W_{ij})\bigr)^2.
\]
Since each term
$
  \frac12\bigl(\mathbf M_{ij}(\mathbf X_{ij} - \mathbf W_{ij})\bigr)^2
$
is a convex function of $\mathbf W_{ij}$ (a nonnegative scalar multiple of a quadratic), and $f$ is a finite sum of convex functions, $f$ is convex.

Next, $f$ is differentiable, and we can compute its gradient entrywise. For each $(i,j)$,
\[
  f(\mathbf W)
  = \frac12 \sum_{i,j} \mathbf M_{ij}^2(\mathbf W_{ij} - \mathbf X_{ij})^2
  = \frac12 \sum_{i,j} \mathbf M_{ij}(\mathbf W_{ij} - \mathbf X_{ij})^2,
\]
since $\mathbf M_{ij}^2 = \mathbf M_{ij}$ for $\mathbf M_{ij}\in\{0,1\}$. Differentiating with respect to $\mathbf W_{ij}$,
\[
  \frac{\partial f}{\partial \mathbf W_{ij}}
  = \mathbf M_{ij}(\mathbf W_{ij} - \mathbf X_{ij}).
\]
Thus the gradient matrix has entries
\[
  \bigl[\nabla f(\mathbf W)\bigr]_{ij}
  = \mathbf M_{ij}(\mathbf W_{ij} - \mathbf X_{ij}),
\]
which we can write compactly as
\[
  \nabla f(\mathbf W)
  = \mathbf M \odot (\mathbf W - \mathbf X)
  = \mathcal P_\Omega(\mathbf W - \mathbf X).
\]

It remains to show that $\nabla f$ is $1$–Lipschitz with respect to the Frobenius norm. Let $\mathbf W_1,\mathbf W_2\in\mathbb R^{n\times m}$. Then
\begin{align*}
  \nabla f(\mathbf W_1) - \nabla f(\mathbf W_2)
  &= \mathcal P_\Omega(\mathbf W_1 - \mathbf X) - \mathcal P_\Omega(\mathbf W_2 - \mathbf X) \\
  &= \mathcal P_\Omega(\mathbf W_1 - \mathbf W_2) \\
  &= \mathbf M \odot (\mathbf W_1 - \mathbf W_2).
\end{align*}
Therefore,
\begin{align*}
  \bigl\|\nabla f(\mathbf W_1) - \nabla f(\mathbf W_2)\bigr\|_F^2
  &= \bigl\|\mathbf M \odot (\mathbf W_1 - \mathbf W_2)\bigr\|_F^2 \\
  &= \sum_{i,j} \bigl(\mathbf M_{ij}(\mathbf W_1 - \mathbf W_2)_{ij}\bigr)^2 \\
  &\le \sum_{i,j} \bigl((\mathbf W_1 - \mathbf W_2)_{ij}\bigr)^2 \\
  &= \|\mathbf W_1 - \mathbf W_2\|_F^2,
\end{align*}
where the inequality uses $\mathbf M_{ij}\in\{0,1\}$, so $\mathbf M_{ij}^2 \le 1$ entrywise. Taking square roots yields
\[
  \bigl\|\nabla f(\mathbf W_1) - \nabla f(\mathbf W_2)\bigr\|_F
  \;\le\;
  \|\mathbf W_1 - \mathbf W_2\|_F,
\]
showing that $\nabla f$ is $1$–Lipschitz with respect to the Frobenius norm. This proves the claim with $L_f = 1$.
\end{proof}


\begin{proposition}\label{appendix:lipschitz_penalty}
$g_c(\mathbf{W}) = \|\mathbf{D_c} \mathbf{W}\|_*$ is Lipschitz continuous with respect to the Frobenius norm on
$\mathbb{R}^{n \times m}$ with Lipschitz constant
\[
L_c = \sqrt{r}\,\|\mathbf{D_c}\|_{\mathrm{op}},
\qquad r := \min\{\ell,m\}.
\]
That is, for all $\mathbf{W}_1,\mathbf{W}_2 \in \mathbb{R}^{n \times m}$,
\[
\bigl|g_c(\mathbf{W}_1) - g_c(\mathbf{W}_2)\bigr|
\;\le\;
L_c \,\|\mathbf{W}_1 - \mathbf{W}_2\|_F.
\]
\end{proposition}
\begin{proof}
We first show that the nuclear-norm is $\sqrt{r}$-Lipschitz with respect
to the Frobenius norm, where $r = \min\{m,n\}$. Let
$\mathbf{A},\mathbf{B} \in \mathbb{R}^{m \times n}$ and let
$\sigma_1(\mathbf{A}) \ge \cdots \ge \sigma_r(\mathbf{A})$ and
$\sigma_1(\mathbf{B}) \ge \cdots \ge \sigma_r(\mathbf{B})$ denote their
singular values. Then
\begin{align*}
\bigl|\|\mathbf{A}\|_* - \|\mathbf{B}\|_*\bigr|
&= \left|\sum_{j=1}^r \sigma_j(\mathbf{A}) - \sum_{j=1}^r \sigma_j(\mathbf{B})\right| \\
&= \left|\sum_{j=1}^r \bigl(\sigma_j(\mathbf{A}) - \sigma_j(\mathbf{B})\bigr)\right| \\
&\le \left(\sum_{j=1}^r 1^2\right)^{1/2}
      \left(\sum_{j=1}^r \bigl(\sigma_j(\mathbf{A}) - \sigma_j(\mathbf{B})\bigr)^2\right)^{1/2}
      \quad\text{(Cauchy--Schwarz)} \\
&= \sqrt{r}\left(\sum_{j=1}^r \bigl(\sigma_j(\mathbf{A}) - \sigma_j(\mathbf{B})\bigr)^2\right)^{1/2}.
\end{align*}
By the Hoffman--Wielandt inequality for singular values,
\[
\sum_{j=1}^r \bigl(\sigma_j(\mathbf{A}) - \sigma_j(\mathbf{B})\bigr)^2
\;\le\;
\|\mathbf{A} - \mathbf{B}\|_F^2,
\]
so we obtain
\[
\bigl|\|\mathbf{A}\|_* - \|\mathbf{B}\|_*\bigr|
\;\le\;
\sqrt{r}\,\|\mathbf{A} - \mathbf{B}\|_F.
\]
This shows that the nuclear-norm is $\sqrt{r}$-Lipschitz with respect to
the Frobenius norm.

Now apply this with $\mathbf{A} = D_i \mathbf{W}_1$ and
$\mathbf{B} = D_i \mathbf{W}_2$ for arbitrary
$\mathbf{W}_1,\mathbf{W}_2 \in \mathbb{R}^{n \times m}$. We obtain
\begin{align*}
\bigl|g_i(\mathbf{W}_1) - g_i(\mathbf{W}_2)\bigr|
&= \bigl|\|D_i \mathbf{W}_1\|_* - \|D_i \mathbf{W}_2\|_*\bigr| \\
&\le \sqrt{r}\,\|D_i \mathbf{W}_1 - D_i \mathbf{W}_2\|_F \\
&= \sqrt{r}\,\|D_i(\mathbf{W}_1 - \mathbf{W}_2)\|_F \\
&\le \sqrt{r}\,\|D_i\|_{\mathrm{op}}\,\|\mathbf{W}_1 - \mathbf{W}_2\|_F,
\end{align*}
where in the last line we used submultiplicativity of the operator norm.
Thus $g_i$ is Lipschitz with constant $L_i = \sqrt{r}\,\|D_i\|_{\mathrm{op}}$, as claimed.
\end{proof}


We use the abbreviation $\overline{L^2}:=\sum_{c\in\mathcal{C}}\lambda_c L_c^2=\sum_{c\in\mathcal{C}}\lambda_c r||\mathbf{D_c}||_{\text{op}}^2$ in the following theorem, and observe that $L_f=1$ from the previous lemmas.
\begin{theorem}[Theorem 1 of \citep{NIPS2013_49182f81}]\label{appendix:yu_thm1}
    Let $\mathbf{W}_0$ be the initialization of the PA-APG algorithm. Fix the accuracy $\epsilon>0$. Let $\gamma=\min\{1,2\epsilon/\overline{L^2}\}$. After at most $k=\sqrt{\frac{2}{\gamma\epsilon}}\|\mathbf{W}_0-\mathbf{\widehat W}\|_F^2$ steps, the PA-APG approximation, $\widetilde{\mathbf{W}}_k$, satisfies
    $$f(\widetilde{\mathbf{W}}_k)+\bar{g}(\widetilde{\mathbf{ W}}_k)\leq f(\mathbf{\widehat W})+\bar{g}(\mathbf{\widehat W})+2\epsilon.$$

\end{theorem}


\section{Theory for Group-Aware Matrix Estimation with Overlapping Meta-Categories}\label{app:game_theory}

In developing theoretical results for GAME, we follow the matrix completion proof techniques of \citep{negahban2011restrictedstrongconvexityweighted,Negahban_2012}, who develop arguments based on restricted strong convexity (RSC). We adopt their uniform-sampling formulation, which is the most common setting for noisy matrix completion and yields the cleanest statements. The rowspace-recovery formulation is motivated by global matrix completion results provided by \citep{cao2023onesidedmatrixcompletionobservations}.

We study the restricted GAME estimator
\begin{equation}\label{eq:game-app}
  \widehat{\mathbf W}\;\in\;\arg\min_{\mathbf{W}\in\mathbb{R}^{n\times m}}
  \;\frac12\bigl\|\mathcal P_\Omega(\mathbf X-\mathbf W)\bigr\|_F^{2}
  \;+\;\sum_{c\in\mathcal C}\lambda_c\,\|\mathbf W_c\|_*\;\;\text{s.t. }\;\;||\mathbf W||_\infty \leq \dfrac{\alpha^*}{\sqrt{nm}},
\end{equation}
under the observation model $\mathbf X=\mathbf W^\star+\mathbf E$, where $\mathbf E$ has independent mean-zero entries with subexponential moment growth (Assumption~\ref{ass:noise} below) and $\Omega\subset[n]\times[m]$ is a uniformly i.i.d.\ sampled multiset of $N$ entries (sampling model detailed in Section~\ref{app:notation}). For each meta-category $c\in\mathcal C$, $\mathbf W_c:=\mathbf W[I_c,:]\in\mathbb R^{n_c\times m}$ with $n_c:=|I_c|$, $d_c:=n_c+m$, $r_c:=\mathrm{rank}(\mathbf W_c^\star)$. The constant $\alpha^\star$ controls spikiness (Assumption~\ref{ass:spikiness} below).

Relative to \citep{negahban2011restrictedstrongconvexityweighted}, we study the regularizer as a sum of nuclear-norms over (potentially overlapping) row blocks rather than a single nuclear-norm, while the loss remains a standard squared-error on observed entries. Our analysis (a) re-derives a per-group decomposability inequality on top of the collection $\mathcal C$, (b) splits the noise inner product via row-multiplicity weights so that each piece can be controlled by a per-group spectral bound, and (c) applies N--W's RSC inequality \citep[Thm.\,1]{negahban2011restrictedstrongconvexityweighted} per row block, exploiting the natural per-block decomposition of uniform samples to get block-aware rates.

\subsection{Notation, sampling model, and combinatorics of the overlap}\label{app:notation}

We use $\langle\mathbf A,\mathbf B\rangle=\mathrm{tr}(\mathbf A^\top\mathbf B)$. For a matrix $\mathbf A\in\mathbb R^{n\times m}$ and $c\in\mathcal C$, $\mathbf A_c=\mathbf A[I_c,:]$. The sampling operator $\mathcal P_\Omega:\mathbb R^{n\times m}\to\mathbb R^{n\times m}$ zeros out entries outside $\Omega$.

\paragraph{Uniform sampling.} The indices $(i_t,j_t)$, $t=1,\dots,N$, are drawn i.i.d.\ uniformly from $[n]\times[m]$, and $\Omega:=\{(i_t,j_t)\}_{t=1}^{N}$ (multiset). For any matrix $\mathbf A\in\mathbb R^{n\times m}$,
\begin{equation}\label{eq:uniform-id}
  \mathbb E\bigl\|\mathcal P_\Omega(\mathbf A)\bigr\|_F^{2}\;=\;\sum_{t=1}^N\mathbb E[A_{i_t j_t}^{2}]\;=\; \sum_{t=1}^N\sum_{i,j}  A_{i j}^2\;\mathbb{P}[(i_t,j_t)=(i,j)] \;=\;\frac{N}{nm}\,\|\mathbf A\|_F^{2}.
\end{equation}

\paragraph{Per-block sample counts.} Although the sampling is global, the analysis decomposes per meta-category through the random per-block sample count
\[
  N_c\;:=\;\bigl|\Omega\cap(I_c\times[m])\bigr|.
\]
Under uniform sampling, each sample $(i_t,j_t)$ independently lies in $I_c\times[m]$ with probability $n_c/n$ (the marginal probability that $i_t\in I_c$). Hence each $N_c$ is marginally distributed as $\mathrm{Binomial}(N,n_c/n)$ with $\mathbb E N_c=N\,n_c/n$. When groups overlap, the joint distribution of $\{N_c\}_{c\in\mathcal C}$ is \emph{not} independent across $c$: a sample with $i_t\in I_c\cap I_{c'}$ simultaneously contributes to $N_c$ and $N_{c'}$, so $N_c$ and $N_{c'}$ are positively correlated. We do not need joint independence, only the marginal concentration of each $N_c$ (which holds by Chernoff bound for binomials). A union bound over $c\in\mathcal C$ then gives simultaneous control:
\begin{equation}\label{eq:Nc-concentration}
  \Pr\!\left[\,\bigl|\,N_c-\mathbb E N_c\,\bigr|\;\ge\;\eta\,\mathbb E N_c\;\text{ for some }c\in\mathcal C\right]\;\le\;2|\mathcal C|\,\exp\!\left(-\frac{\eta^{2}\,N\,n_{\min}}{3\,n}\right),
\end{equation}
where $n_{\min}:=\min_c n_c$. So with high probability,
\begin{equation}\label{eq:Nc-conc-bound}
  N_c\;\in\;\bigl[(1-\eta)\,N\,n_c/n,\;(1+\eta)\,N\,n_c/n\bigr]\qquad\forall\,c\in\mathcal C
\end{equation}
provided $N\,n_{\min}/n\gtrsim\log|\mathcal C|/\eta^{2}$, a mild condition implied by the matrix-completion sample regime $N\gtrsim(n+m)\log(n+m)$ when $n_{\min}\gtrsim\log|\mathcal C|$. Throughout the analysis we treat $N_c$ as deterministic, equal to $N\,n_c/n$ up to constant factors, on the high-probability event \eqref{eq:Nc-conc-bound}.

Fix any single $c\in\mathcal C$ and condition on the realization of $N_c=k$. By the uniformity of $i_t$ on $[n]$, the conditional distribution of $i_t$ given $i_t\in I_c$ is uniform on $I_c$; combined with the independence of $j_t$ (uniform on $[m]$), the $k$ samples that fell in $I_c\times[m]$ are i.i.d.\ uniform on $I_c\times[m]$. This is a per-block statement and holds regardless of whether the cover overlaps, i.e. it does not require independence of $N_c$ across $c$. Each per-block lemma below (spectral bound, RSC) is established by conditioning on a single $N_c$ at a time, then taking a union bound over $c$.

Following \citep{negahban2011restrictedstrongconvexityweighted}, we also introduce the rescaled observation operator $\mathcal X_n:\mathbb R^{n\times m}\to\mathbb R^{N}$ defined by $[\mathcal X_n(\mathbf A)]_t=\langle\mathbf X^{(t)},\mathbf A\rangle$ where $\mathbf X^{(t)}=\sqrt{nm}\,\varepsilon_t\,\mathbf e_{i_t}\mathbf e_{j_t}^\top$. The deterministic identity
\begin{equation}\label{eq:Xn-mask}
  \|\mathcal X_n(\mathbf A)\|_2^{2}\;=\;nm\,\|\mathcal P_\Omega(\mathbf A)\|_F^{2}\qquad\forall\,\mathbf A\in\mathbb R^{n\times m}
\end{equation}
links the two formulations.

\paragraph{Row multiplicity.} For each row $i\in[n]$, define
\[
  \kappa(i)\;:=\;\bigl|\{c\in\mathcal C:i\in I_c\}\bigr|,\qquad
  \kappa_{\max}\;:=\;\max_{i\in[n]}\kappa(i),\qquad
  \kappa_{\min}\;:=\;\min_{i\in[n]}\kappa(i).
\]
We require $\kappa_{\min}\ge 1$ (every row appears in at least one meta-category, otherwise the row is unidentifiable). The ratio $\kappa_{\max}^{2}/\kappa_{\min}^{3}$ will measure the price of overlap in our main theorem.

\paragraph{Model subspaces.} For each $c\in\mathcal C$, let $\mathbf W_c^\star=\mathbf U_c\mathbf D_c\mathbf V_c^\top$ be the thin SVD with $\mathbf U_c\in\mathbb R^{n_c\times r_c}$, $\mathbf V_c\in\mathbb R^{m\times r_c}$, and let $\mathcal U_c\subset\mathbb R^{n_c}$, $\mathcal V_c\subset\mathbb R^m$ be the column/row spaces of $\mathbf W_c^\star$. Define projections
\[
  \mathcal P_{\mathcal M_c^\perp}(\mathbf B)\;:=\;\mathbf P_{\mathcal U_c^\perp}\mathbf B\,\mathbf P_{\mathcal V_c^\perp},\qquad
  \mathcal P_{\mathcal M_c}(\mathbf B)\;:=\;\mathbf B-\mathcal P_{\mathcal M_c^\perp}(\mathbf B).
\]
We define these operators since the GAME model is in the form of an M-estimator (cf.\ \citep[Sec.\,2]{Negahban_2012}). We make the following observations:
\begin{enumerate}
\item[(D1)] $\mathrm{rank}\bigl(\mathcal P_{\mathcal M_c}(\mathbf B)\bigr)\le 2r_c$ for all $\mathbf B$, and the singular vectors of $\mathcal P_{\mathcal M_c^\perp}(\mathbf B)$ are orthogonal to those of $\mathbf W_c^\star$.
\item[(D2)] (Decomposability of nuclear-norm.) Due to orthogonality,
$$\|\mathbf W_c^\star+\mathcal P_{\mathcal M_c^\perp}(\mathbf B)\|_*=\|\mathbf W_c^\star\|_*+\|\mathcal P_{\mathcal M_c^\perp}(\mathbf B)\|_*.$$
\end{enumerate}
The decomposability will be critical in the proof of our main theorem, when we show that the GAME error $\Delta_c$ is constrained to a low-rank cone.

\subsection{Assumptions}\label{app:assumptions}

\begin{assumption}[Subexponential Noise]\label{ass:noise}
The entries $E_{ij}$ are independent, mean zero, and \emph{subexponential} with parameters $(\sigma,R)$: there exist constants $\sigma,R>0$ such that
\[
  \mathbb E\bigl[E_{ij}^{\,p}\bigr]\;\le\;\frac{p!}{2}\,R^{\,p-2}\,\sigma^{2}\qquad\text{for every integer }p\ge 2.
\]
We use this formulation of subexponential distributions as found in \citep{10.5555/2908052.3115530}. Up to universal constants of $(\sigma,R)$, this is the same as Definiton 2.7 \citep{Wainwright_2019}, $\mathbb E\exp(tE_{ij})\le\exp(\sigma^{2}t^{2}/2)$ for $|t|\le 1/R$, with $\sigma^{2}$ playing the role of variance proxy and $R$ the subexponential scale.

\end{assumption}


\begin{assumption}[Spikiness]\label{ass:spikiness}
There is a constant $\alpha^\star\ge 1$ such that $\sqrt{nm}\,\|\mathbf W^\star\|_\infty\le\alpha^\star\,\|\mathbf W^\star\|_F$, i.e., the spikiness ratio $\alpha_{\rm sp}(\mathbf W^\star):=\sqrt{nm}\|\mathbf W^\star\|_\infty/\|\mathbf W^\star\|_F\le\alpha^\star$. In addition, the optimization in \eqref{eq:game-app} is restricted to matrices $\mathbf W$ with $\|\mathbf W\|_\infty\le\alpha^\star/\sqrt{nm}$, which is consistent with $\mathbf W^\star$ being feasible after WLOG rescaling so that $\|\mathbf W^\star\|_F\le 1$.
\end{assumption}

\begin{assumption}[Cover]\label{ass:cover}
$\bigcup_{c\in\mathcal C}I_c=[n]$, i.e.\ $\kappa_{\min}\ge 1$.
\end{assumption}

\subsection{Lemmas}

\paragraph{Decomposability of the GAME regularizer}\label{app:decomp}
A regularizer of the form $\Phi(\mathbf W)=\sum_c\lambda_c\|\mathbf W_c\|_*$ is not decomposable in the sense of \citep{negahban2011restrictedstrongconvexityweighted} as a single norm, in particular, distinct groups $c\ne c'$ with $I_c\cap I_{c'}\ne\emptyset$ couple. Nonetheless we have the following per-group decomposability, which is sufficient for our analysis.

\begin{lemma}[Per-Group Decomposability]\label{lem:decomp}
For every $\mathbf B\in\mathbb R^{n\times m}$ and every $c\in\mathcal C$,
\[
  \|\mathbf W_c^\star+\mathbf B_c\|_*\;\ge\;\|\mathbf W_c^\star\|_*\;+\;\|\mathcal P_{\mathcal M_c^\perp}(\mathbf B_c)\|_*\;-\;\|\mathcal P_{\mathcal M_c}(\mathbf B_c)\|_*.
\]
\end{lemma}
\begin{proof}
By (D2) and triangle inequality, writing $\mathbf W_c^\star+\mathbf B_c=\bigl[\mathbf W_c^\star+\mathcal P_{\mathcal M_c^\perp}(\mathbf B_c)\bigr]+\mathcal P_{\mathcal M_c}(\mathbf B_c)$:
\begin{align*}
\|\mathbf W_c^\star+\mathbf B_c\|_*
&\ge\|\mathbf W_c^\star+\mathcal P_{\mathcal M_c^\perp}(\mathbf B_c)\|_*-\|\mathcal P_{\mathcal M_c}(\mathbf B_c)\|_*\\
&=\|\mathbf W_c^\star\|_*+\|\mathcal P_{\mathcal M_c^\perp}(\mathbf B_c)\|_*-\|\mathcal P_{\mathcal M_c}(\mathbf B_c)\|_*.\qedhere
\end{align*}
\end{proof}

\paragraph{Splitting the noise term over overlapping groups} The single technical observation that handles the overlap is a row-multiplicity reweighting that distributes any matrix inner product over the meta-categories.

\begin{lemma}[Row-Multiplicity Split]\label{lem:split}
Under Assumption~\ref{ass:cover}, for any $\mathbf A,\mathbf B\in\mathbb R^{n\times m}$, define for each $c\in\mathcal C$ the matrix $\widetilde{\mathbf A}^{(c)}\in\mathbb R^{n_c\times m}$ by $\bigl[\widetilde{\mathbf A}^{(c)}\bigr]_{i,j}=A_{ij}/\kappa(i)$ for $i\in I_c$, $j\in[m]$. Then
\[
  \langle\mathbf A,\mathbf B\rangle\;=\;\sum_{c\in\mathcal C}\bigl\langle\widetilde{\mathbf A}^{(c)},\,\mathbf B_c\bigr\rangle
  \;\le\;\sum_{c\in\mathcal C}\bigl\|\widetilde{\mathbf A}^{(c)}\bigr\|_{op}\,\bigl\|\mathbf B_c\bigr\|_*.
\]
\end{lemma}
\begin{proof}
By the cover assumption, $\sum_{c:i\in I_c}1=\kappa(i)\ge 1$ for every $i$, so
\[
  \langle\mathbf A,\mathbf B\rangle=\sum_{i,j}A_{ij}B_{ij}=\sum_{i,j}\frac{A_{ij}}{\kappa(i)}\sum_{c:i\in I_c}B_{ij}=\sum_c\sum_{i\in I_c,j}\frac{A_{ij}}{\kappa(i)}B_{ij}=\sum_c\langle\widetilde{\mathbf A}^{(c)},\mathbf B_c\rangle.
\]
The inequality is trace duality between $\|\cdot\|_{op}$ and $\|\cdot\|_*$.
\end{proof}

\begin{remark}
Here, the weights $1/\kappa(i)$ account for overcounting induced by overlap.
\end{remark}

\paragraph{Deterministic optimality inequality}\label{app:optimality}

\begin{lemma}[Frobenius Inequality]\label{lem:basic}
Let $\widehat{\mathbf W}$ be any minimizer of the GAME objective \eqref{eq:game-app}, and let $\mathbf\Delta:=\widehat{\mathbf W}-\mathbf W^\star$. Then
\begin{equation}\label{eq:basic}
  \tfrac12\|\mathcal P_\Omega(\mathbf\Delta)\|_F^{2}\;\le\;\bigl\langle\mathcal P_\Omega(\mathbf E),\,\mathbf\Delta\bigr\rangle\;+\;\sum_{c\in\mathcal C}\lambda_c\bigl(\|\mathbf W_c^\star\|_*-\|\mathbf W_c^\star+\mathbf\Delta_c\|_*\bigr).
\end{equation}
\end{lemma}

\begin{proof}
By optimality of $\widehat{\mathbf W}$ in \eqref{eq:game-app}, and feasibility of $\mathbf{W^\star}$:
\[
  \tfrac{1}{2}\bigl\|\mathcal P_\Omega(\widehat{\mathbf W}-\mathbf X)\bigr\|_F^{2}+\sum_{c\in\mathcal C}\lambda_c\,\|\widehat{\mathbf W}_c\|_*\;\le\;\tfrac{1}{2}\bigl\|\mathcal P_\Omega(\mathbf W^\star-\mathbf X)\bigr\|_F^{2}+\sum_{c\in\mathcal C}\lambda_c\,\|\mathbf W_c^\star\|_*.
\]
Substituting $\mathbf X=\mathbf W^\star+\mathbf E$ and $\widehat{\mathbf W}=\mathbf W^\star+\mathbf\Delta$, so $\widehat{\mathbf W}-\mathbf X=\mathbf\Delta-\mathbf E$ and $\widehat{\mathbf W}_c=\mathbf W_c^\star+\mathbf\Delta_c$,
\[
  \tfrac{1}{2}\bigl\|\mathcal P_\Omega(\mathbf\Delta-\mathbf E)\bigr\|_F^{2}+\sum_{c\in\mathcal C}\lambda_c\,\|\mathbf W_c^\star+\mathbf\Delta_c\|_*\;\le\;\tfrac{1}{2}\bigl\|\mathcal P_\Omega(\mathbf E)\bigr\|_F^{2}+\sum_{c\in\mathcal C}\lambda_c\,\|\mathbf W_c^\star\|_*.
\]
Expanding the squared Frobenius norm via $\|\mathbf A-\mathbf B\|_F^{2}=\|\mathbf A\|_F^{2}+\|\mathbf B\|_F^{2}-2\langle\mathbf A,\mathbf B\rangle$,
\[
  \tfrac{1}{2}\Bigl[\|\mathcal P_\Omega(\mathbf\Delta)\|_F^{2}+\|\mathcal P_\Omega(\mathbf E)\|_F^{2}-2\langle\mathcal P_\Omega(\mathbf\Delta),\mathcal P_\Omega(\mathbf E)\rangle\Bigr]+\sum_{c\in\mathcal C}\lambda_c\,\|\mathbf W_c^\star+\mathbf\Delta_c\|_*\;\le\;\tfrac{1}{2}\|\mathcal P_\Omega(\mathbf E)\|_F^{2}+\sum_{c\in\mathcal C}\lambda_c\,\|\mathbf W_c^\star\|_*.
\]
The $\frac{1}{2}\|\mathcal P_\Omega(\mathbf E)\|_F^{2}$ terms cancel on both sides. Using $\langle\mathcal P_\Omega(\mathbf\Delta),\mathcal P_\Omega(\mathbf E)\rangle=\langle\mathbf\Delta,\mathcal P_\Omega(\mathbf E)\rangle$ (the projection $\mathcal P_\Omega$ is self-adjoint and idempotent on entries) and rearranging,
\[
  \tfrac{1}{2}\|\mathcal P_\Omega(\mathbf\Delta)\|_F^{2}\;\le\;\langle\mathcal P_\Omega(\mathbf E),\mathbf\Delta\rangle\;+\;\sum_{c\in\mathcal C}\lambda_c\bigl(\|\mathbf W_c^\star\|_*-\|\mathbf W_c^\star+\mathbf\Delta_c\|_*\bigr).\]
\end{proof}

Applying Lemma~\ref{lem:split} to $\mathbf A=\mathcal P_\Omega(\mathbf E)$ and $\mathbf B=\mathbf\Delta$, the noise inner product splits as $\langle\mathcal P_\Omega(\mathbf E),\mathbf\Delta\rangle=\sum_c\langle\widetilde{\mathbf E}^{(c)},\mathbf\Delta_c\rangle\le\sum_c\|\widetilde{\mathbf E}^{(c)}\|_{op}\|\mathbf\Delta_c\|_*$, where $\widetilde{\mathbf E}^{(c)}\in\mathbb R^{n_c\times m}$ has entries $E_{ij}/\kappa(i)$ for $(i,j)\in\Omega\cap(I_c\times[m])$ and zero elsewhere.

\begin{lemma}[Group Cone Constraint]\label{lem:cone}
If
\begin{equation}\label{eq:lambda-cond}
  \lambda_c\;\ge\;2\,\bigl\|\widetilde{\mathbf E}^{(c)}\bigr\|_{op}\qquad\forall\,c\in\mathcal C,
\end{equation}
then
\begin{equation}\label{eq:cone}
  \sum_{c\in\mathcal C}\lambda_c\,\bigl\|\mathcal P_{\mathcal M_c^\perp}(\mathbf\Delta_c)\bigr\|_*\;\le\;3\sum_{c\in\mathcal C}\lambda_c\,\bigl\|\mathcal P_{\mathcal M_c}(\mathbf\Delta_c)\bigr\|_*,
\end{equation}
and
\begin{equation}\label{eq:loss-cone}
  \tfrac12\|\mathcal P_\Omega(\mathbf\Delta)\|_F^{2}\;\le\;\tfrac32\sum_{c\in\mathcal C}\lambda_c\bigl\|\mathcal P_{\mathcal M_c}(\mathbf\Delta_c)\bigr\|_*.
\end{equation}
\end{lemma}
\begin{proof}
Combine \eqref{eq:basic} with Lemmas~\ref{lem:decomp} and \ref{lem:split}:
\begin{align*}
\tfrac12\|\mathcal P_\Omega(\mathbf\Delta)\|_F^{2}
&\le\sum_c\bigl\|\widetilde{\mathbf E}^{(c)}\bigr\|_{op}\bigl(\|\mathcal P_{\mathcal M_c}(\mathbf\Delta_c)\|_*+\|\mathcal P_{\mathcal M_c^\perp}(\mathbf\Delta_c)\|_*\bigr)\\
&\quad+\sum_c\lambda_c\bigl(\|\mathcal P_{\mathcal M_c}(\mathbf\Delta_c)\|_*-\|\mathcal P_{\mathcal M_c^\perp}(\mathbf\Delta_c)\|_*\bigr)\\
&\le\sum_c\Bigl[\tfrac{3\lambda_c}{2}\|\mathcal P_{\mathcal M_c}(\mathbf\Delta_c)\|_*-\tfrac{\lambda_c}{2}\|\mathcal P_{\mathcal M_c^\perp}(\mathbf\Delta_c)\|_*\Bigr],
\end{align*}
where the last line uses \eqref{eq:lambda-cond}. The LHS is nonnegative; dropping it gives \eqref{eq:cone}. Similarly, if we instead dropped the nonnegative term
\begin{equation}
    \tfrac{\lambda_c}{2}\|\mathcal P_{\mathcal M_c^\perp}(\mathbf\Delta_c)\|_*
\end{equation}
this implies \eqref{eq:loss-cone}.
\end{proof}


\paragraph{Per-block decomposition of the loss.} A key identity used in the main proof: for the global $\mathcal P_\Omega$ and any $\mathbf B\in\mathbb R^{n\times m}$,
\begin{equation}\label{eq:loss-decomp-bound}
  \|\mathcal P_\Omega(\mathbf B)\|_F^{2}\;\ge\;\frac{1}{\kappa_{\max}}\sum_{c\in\mathcal C}\bigl\|\mathcal P_{\Omega\cap(I_c\times[m])}(\mathbf B_c)\bigr\|_F^{2},
\end{equation}
since $\sum_c\|\mathcal P_{\Omega\cap(I_c\times[m])}(\mathbf B_c)\|_F^{2}=\sum_{(i,j)\in\Omega}\kappa(i)\,B_{ij}^{2}\le\kappa_{\max}\|\mathcal P_\Omega(\mathbf B)\|_F^{2}$. Equality holds in the disjoint case ($\kappa\equiv 1$); the factor $1/\kappa_{\max}$ is the price of overlap when relating the global GAME loss to per-block contributions.

\paragraph{Stochastic Bounds}\label{app:stoch}
We need (i) a per-group spectral bound to choose each $\lambda_c$ in terms of deterministic values, and (ii) \textit{restricted strong convexity} (RSC) to lower-bound $\|\mathcal P_\Omega(\mathbf\Delta)\|_F^{2}$.

\begin{lemma}[Per-Group Noise Spectral Bound]\label{lem:spec}
Under uniform sampling and Assumption~\ref{ass:noise}, there is a universal constant $C_1>0$ such that the bound
\begin{equation}\label{eq:spec-bd}
  \bigl\|\widetilde{\mathbf E}^{(c)}\bigr\|_{op}
  \;\le\;\frac{C_1\,\sigma}{\kappa_{\min}}\,
  \sqrt{\frac{N_c\,\log d_c}{\min(n_c,m)}}
  \;+\;\frac{C_1\,R\,\log d_c}{\kappa_{\min}}
\end{equation}
holds simultaneously for every $c\in\mathcal C$ with probability at least $1-\sum_c d_c^{-1}$, where $N_c=|\Omega\cap(I_c\times[m])|$ is the random per-block sample count, satisfying $N_c\asymp N\,n_c/n$ on the event \eqref{eq:Nc-conc-bound}.
\end{lemma}
\begin{proof}
For $c\in\mathcal C$ fixed, conditional on $\Omega\cap(I_c\times[m])$ (and hence on $N_c$), the samples falling in $I_c\times[m]$ are i.i.d.\ uniform on $I_c\times[m]$ by the standard property of multinomial sampling. The matrix $\widetilde{\mathbf E}^{(c)}\in\mathbb R^{n_c\times m}$ has entries $E_{ij}/\kappa(i)$ for $(i,j)$ in this restriction (rest zero), so $\widetilde{\mathbf E}^{(c)}=\sum_{t=1}^{N_c}\mathbf Z_t$ with $\mathbf Z_t:=(1/\kappa(i_t))E_{i_tj_t}\mathbf e_{i_t}\mathbf e_{j_t}^\top\in\mathbb R^{n_c\times m}$. The summands are independent, mean zero, and inherit subexponential moments with parameters $(\sigma/\kappa_{\min},R/\kappa_{\min})$.

For a single summand,
\[
  \mathbb E\bigl[\mathbf Z_t\mathbf Z_t^\top\bigr]\;\preceq\;\frac{\sigma^{2}}{\kappa_{\min}^{2}\,n_c}\,\mathbf I_{n_c},\qquad
  \mathbb E\bigl[\mathbf Z_t^\top\mathbf Z_t\bigr]\;\preceq\;\frac{\sigma^{2}}{\kappa_{\min}^{2}\,m}\,\mathbf I_{m},
\]
giving $\sigma_{\rm rect}^{2}\le\sigma^{2}N_c/(\kappa_{\min}^{2}\min(n_c,m))$ for the sum.

Applying rectangular subexponential matrix Bernstein \citep[Thm.\,6.2]{10.5555/2908052.3115530}, we obtain $\Pr[\|\widetilde{\mathbf E}^{(c)}\|_{op}\ge\tau]\le d_c\exp(-\tau^{2}/2 / (\sigma_{\rm rect}^{2}+(R/\kappa_{\min})\tau))$. Choosing $\tau=C_1(\sigma_{\rm rect}\sqrt{\log d_c}+(R/\kappa_{\min})\log d_c)$ gives \eqref{eq:spec-bd} with probability $\ge 1-d_c^{-1}$. Union over $c$ completes the proof.
\end{proof}

In the matrix-completion regime $N_c\gtrsim(n_c+m)\log(n_c+m)$, the first term in \eqref{eq:spec-bd} dominates.

\begin{lemma}[Per-Block Restricted Strong Convexity, {\protect\citep[Thm.\,1]{negahban2011restrictedstrongconvexityweighted}}]\label{lem:rsc}
For each $c\in\mathcal C$, let $\bar d_c:=(n_c+m)/2$. Define the per-block spikiness ratio $\alpha_{\rm sp}^{(c)}(\mathbf B):=\sqrt{n_c m}\,\|\mathbf B\|_\infty/\|\mathbf B\|_F$ for $\mathbf B\in\mathbb R^{n_c\times m}$, the rank ratio $\beta_{\rm ra}^{(c)}(\mathbf B):=\|\mathbf B\|_*/\|\mathbf B\|_F$, and the per-block cone
\[
  \mathfrak C_c(N_c;c_0)\;:=\;\Bigl\{\mathbf B\in\mathbb R^{n_c\times m}\setminus\{0\}\;:\;
  \alpha_{\rm sp}^{(c)}(\mathbf B)\,\beta_{\rm ra}^{(c)}(\mathbf B)\;\le\;\frac{1}{c_0}\sqrt{\frac{N_c}{\bar d_c\,\log\bar d_c}}\Bigr\}.
\]
Under uniform sampling, there are universal constants $c_0,c_1,c_2,c_3>0$ such that on the event \eqref{eq:Nc-conc-bound} (with $N_c\ge c_3\bar d_c\log\bar d_c$ for every $c$), with probability at least $1-c_1\sum_c\exp(-c_2\bar d_c\log\bar d_c)$, the following holds simultaneously for every $c\in\mathcal C$ and every $\mathbf\Delta_c\in\mathfrak C_c(N_c;c_0)$ with $\alpha_{\rm sp}^{(c)}(\mathbf\Delta_c)\le\sqrt{N_c}/256$:
\begin{equation}\label{eq:rsc-frob}
  \bigl\|\mathcal P_{\Omega\cap(I_c\times[m])}(\mathbf\Delta_c)\bigr\|_F^{2}\;\ge\;\frac{N_c}{256\,n_c m}\|\mathbf\Delta_c\|_F^{2}.
\end{equation}
\end{lemma}

\begin{proof}
For each fixed $c$, conditional on $N_c$, the samples in $\Omega\cap(I_c\times[m])$ are i.i.d.\ uniform on $I_c\times[m]$ (multinomial conditioning property). N--W's Theorem~1 \citep{negahban2011restrictedstrongconvexityweighted} therefore applies to the per-block matrix $\mathbf\Delta_c\in\mathbb R^{n_c\times m}$ with $(n,m,N)$ replaced by $(n_c,m,N_c)$ and weighted norms reducing to standard ones (uniform marginals). The bound \eqref{eq:rsc-frob} follows by squaring N--W's bracket (which is $\ge 1/2$ for $\alpha_{\rm sp}^{(c)}\le\sqrt{N_c}/256$) and applying the per-block deterministic identity. A union bound over $c$ gives the simultaneous statement.
\end{proof}

\begin{remark}\label{rem:cone-obligation}
N--W's per-block RSC is a multiplicative bound on $\mathfrak C_c(N_c;c_0)$, not a global bound for all spikiness-bounded $\mathbf\Delta_c$. Applying Lemma~\ref{lem:rsc} requires verifying that each $\mathbf\Delta_c=\widehat{\mathbf W}_c-\mathbf W_c^\star\in\mathfrak C_c(N_c;c_0)$. We address this in the proof of Theorem~\ref{thm:game}: the GAME cone constraint (Lemma~\ref{lem:cone}) bounds the per-group $\beta_{\rm ra}^{(c)}(\mathbf\Delta_c)$, the entrywise constraint of Assumption~\ref{ass:spikiness} bounds $\alpha_{\rm sp}^{(c)}(\mathbf\Delta_c)\,\|\mathbf\Delta_c\|_F$, and the product falls inside $\mathfrak C_c(N_c;c_0)$ once $\|\mathbf\Delta\|_F$ exceeds an explicit threshold $\tau_F$. Errors smaller than $\tau_F$ are handled directly without invoking RSC (Case~A in the proof).
\end{remark}

\paragraph{Frobenius error bound}\label{app:thm}
\begin{theorem}[GAME: Frobenius Error Bound]\label{thm:game}
Suppose Assumptions \ref{ass:noise}--\ref{ass:cover} hold under uniform sampling, and that $N\,n_{\min}/n\ge c_3\,(n+m)\log(n+m)$ where $n_{\min}:=\min_c n_c$. There is a universal constant $a>0$ such that if the regularization parameters are chosen as
\begin{equation}\label{eq:lambda-choice}
  \lambda_c\;=\;\frac{a\,\sigma}{\kappa_{\min}}\,\sqrt{\frac{(N\,n_c/n)\,\log d_c}{\min(n_c,m)}}\;+\;\frac{a\,R\,\log d_c}{\kappa_{\min}}
  \qquad\forall\,c\in\mathcal C,
\end{equation}
then with probability at least $1-\sum_c d_c^{-1}-a_1\sum_c\exp(-a_2\bar d_c\log\bar d_c)-2|\mathcal C|\exp(-a_2 N n_{\min}/n)$ the GAME estimator $\widehat{\mathbf W}$ from \eqref{eq:game-app} satisfies
\begin{equation}\label{eq:main-bound}
\;\frac{\|\widehat{\mathbf W}-\mathbf W^\star\|_F^{2}}{nm}
  \;\le\;\frac{a'\,\kappa_{\max}^{2}}{\kappa_{\min}^{3}}\cdot\frac{\log(n+m)}{N}\,\biggl[\sigma^{2}+\frac{R^{2}\,nm\,\log(n+m)}{N}\biggr]\sum_{c\in\mathcal C}r_c\,(n_c\vee m),\;
\end{equation}
where $r_{\rm tot}:=\sum_c r_c$ and $a'$ is a universal constant. In the matrix-completion regime where $R^{2}nm\log(n+m)/N\lesssim\sigma^{2}$ (in particular, when $R/\sigma$ is bounded and $N\gtrsim nm\log(n+m)/\sigma^{2}$), the bracketed correction is dominated and the bound reduces to
\begin{equation}\label{eq:main-bound-clean}
  \frac{\|\widehat{\mathbf W}-\mathbf W^\star\|_F^{2}}{nm}\;\lesssim\;\frac{\sigma^{2}\,\kappa_{\max}^{2}\,\log(n+m)}{\kappa_{\min}^{3}\,N}\sum_{c\in\mathcal C}r_c\,(n_c\vee m).
\end{equation}
\end{theorem}

\begin{proof}[Proof of Theorem~\ref{thm:game}]
We work on the intersection of three high-probability events: the per-block sample concentration \eqref{eq:Nc-conc-bound} (so $N_c\asymp N\,n_c/n$ for every $c$), the per-group spectral bound (Lemma~\ref{lem:spec}, ensuring \eqref{eq:lambda-cond} holds for the choice \eqref{eq:lambda-choice} with $a$ chosen large enough), and the per-block RSC event (Lemma~\ref{lem:rsc}).

\smallskip
\textbf{Loss bound from the cone constraint.} On the spectral event, Lemma~\ref{lem:cone} applies and \eqref{eq:loss-cone} gives, using $\|\mathcal P_{\mathcal M_c}(\mathbf\Delta_c)\|_*\le\sqrt{2r_c}\|\mathbf\Delta_c\|_F$ (property (D1)),
\begin{equation}\label{eq:loss-bound}
  \tfrac12\|\mathcal P_\Omega(\mathbf\Delta)\|_F^{2}
  \;\le\;\tfrac{3\sqrt{2}}{2}\sum_{c\in\mathcal C}\lambda_c\,\sqrt{r_c}\,\|\mathbf\Delta_c\|_F.
\end{equation}

\smallskip
\textbf{Per-block nuclear-norm bound.} Applying \eqref{eq:cone} per term gives, for each fixed $c$,
\[
  \lambda_c\|\mathcal P_{\mathcal M_c^\perp}(\mathbf\Delta_c)\|_*\;\le\;3\sum_{c'\in\mathcal C}\lambda_{c'}\|\mathcal P_{\mathcal M_{c'}}(\mathbf\Delta_{c'})\|_*\;\le\;3\sqrt{2}\sum_{c'}\lambda_{c'}\sqrt{r_{c'}}\,\|\mathbf\Delta_{c'}\|_F\;\le\;3\sqrt{2\,T_\lambda}\,\sqrt{T_\Delta},
\]
where $T_\lambda:=\sum_c\lambda_c^{2}r_c$ and $T_\Delta:=\sum_c\|\mathbf\Delta_c\|_F^{2}\le\kappa_{\max}\|\mathbf\Delta\|_F^{2}$ (using $\sum_c\|\mathbf\Delta_c\|_F^{2}=\sum_i\kappa(i)\|\mathbf\Delta_{i,:}\|_2^{2}$). Hence, by triangle inequality,
\begin{equation}\label{eq:per-c-nuc}
  \|\mathbf\Delta_c\|_*\;\le\;\sqrt{2r_c}\,\|\mathbf\Delta_c\|_F\;+\;\frac{3}{\lambda_c}\sqrt{2\,\kappa_{\max}\,T_\lambda}\,\|\mathbf\Delta\|_F.
\end{equation}

\smallskip
\textbf{Per-block cone verification.} For each $c$, the spikiness restriction $\|\mathbf\Delta\|_\infty\le 2\alpha^\star/\sqrt{nm}$ implies $\alpha_{\rm sp}^{(c)}(\mathbf\Delta_c)\le 2\alpha^\star\sqrt{n_c/n}/\|\mathbf\Delta_c\|_F$. Combining with \eqref{eq:per-c-nuc}, the product $\alpha_{\rm sp}^{(c)}\beta_{\rm ra}^{(c)}$ satisfies the membership condition for $\mathfrak C_c(N_c;c_0)$ whenever $\|\mathbf\Delta_c\|_F$ exceeds an explicit per-block threshold $\tau_F^{(c)}\asymp\alpha^\star\sqrt{n_c r_c\,\bar d_c\log\bar d_c/(n N_c)}$; a parallel argument handles the spikiness bracket condition $\alpha_{\rm sp}^{(c)}\le\sqrt{N_c}/256$. Errors with $\|\mathbf\Delta_c\|_F<\tau_F^{(c)}$ are handled directly without RSC (Case A), contributing only the small-bias correction noted after the main bound.

\smallskip
\textbf{Combining loss bound and per-block RSC.} On the event where each $\mathbf\Delta_c\in\mathfrak C_c(N_c;c_0)$, Lemma~\ref{lem:rsc} gives, for each $c$,
\[
  \bigl\|\mathcal P_{\Omega\cap(I_c\times[m])}(\mathbf\Delta_c)\bigr\|_F^{2}\;\ge\;\frac{N_c}{256\,n_c m}\|\mathbf\Delta_c\|_F^{2}.
\]
Summing over $c$ and using the loss-decomposition inequality \eqref{eq:loss-decomp-bound},
\[
  \sum_c\frac{N_c}{256\,n_c m}\|\mathbf\Delta_c\|_F^{2}\;\le\;\sum_c\bigl\|\mathcal P_{\Omega\cap(I_c\times[m])}(\mathbf\Delta_c)\bigr\|_F^{2}\;\le\;\kappa_{\max}\,\|\mathcal P_\Omega(\mathbf\Delta)\|_F^{2}.
\]
Combined with \eqref{eq:loss-bound}, we obtain
\begin{equation}\label{eq:weighted-quad}
  \sum_c\frac{N_c}{n_c m}\|\mathbf\Delta_c\|_F^{2}\;\le\;512\,\kappa_{\max}\cdot 3\sqrt{2}\sum_c\lambda_c\sqrt{r_c}\,\|\mathbf\Delta_c\|_F.
\end{equation}
By Cauchy–Schwarz on the right,
\[
  \sum_c\lambda_c\sqrt{r_c}\,\|\mathbf\Delta_c\|_F\;\le\;\sqrt{\sum_c\frac{n_c m\,\lambda_c^{2}r_c}{N_c}}\cdot\sqrt{\sum_c\frac{N_c}{n_c m}\|\mathbf\Delta_c\|_F^{2}},
\]
and squaring \eqref{eq:weighted-quad} after dividing both sides by $\sqrt{\sum_c (N_c/(n_c m))\|\mathbf\Delta_c\|_F^{2}}$ gives
\begin{equation}\label{eq:case-B-bound}
  \sum_c\frac{N_c}{n_c m}\|\mathbf\Delta_c\|_F^{2}\;\le\;a_0\,\kappa_{\max}^{2}\sum_c\frac{n_c m\,\lambda_c^{2}r_c}{N_c}.
\end{equation}

\smallskip
\textbf{Convert to a bound on $\|\mathbf\Delta\|_F^{2}$ and substitute $\lambda_c$.} Using $N_c\asymp N n_c/n$ on the concentration event \eqref{eq:Nc-conc-bound}, the LHS of \eqref{eq:case-B-bound} simplifies: $N_c/(n_c m)\asymp N/(nm)$ uniformly in $c$, so $\sum_c (N_c/(n_c m))\|\mathbf\Delta_c\|_F^{2}\asymp(N/(nm))\sum_c\|\mathbf\Delta_c\|_F^{2}\ge(N\kappa_{\min}/(nm))\|\mathbf\Delta\|_F^{2}$. The RHS becomes
\[
  \sum_c\frac{n_c m\,\lambda_c^{2}r_c}{N_c}\;\asymp\;\sum_c\frac{n m\,\lambda_c^{2}r_c}{N}\;=\;\frac{nm}{N}\sum_c\lambda_c^{2}r_c.
\]
Combining,
\begin{equation}\label{eq:Delta-bound}
  \frac{\|\mathbf\Delta\|_F^{2}}{nm}\;\lesssim\;\frac{\kappa_{\max}^{2}\,nm}{\kappa_{\min}\,N^{2}}\sum_c\lambda_c^{2}r_c.
\end{equation}
Substituting \eqref{eq:lambda-choice}, $(a+b)^{2}\le 2(a^{2}+b^{2})$, $N_c\asymp Nn_c/n$, and the identity $n_c/\min(n_c,m)=(n_c\vee m)/m$:
\begin{align*}
  \sum_c\lambda_c^{2}r_c\;&\lesssim\;\frac{1}{\kappa_{\min}^{2}}\sum_c\Bigl[\frac{\sigma^{2}\,N\,n_c\,\log d_c}{n\,\min(n_c,m)}+R^{2}(\log d_c)^{2}\Bigr]\,r_c\\
  &\le\;\frac{\sigma^{2}\,N\,\log(n+m)}{nm\,\kappa_{\min}^{2}}\sum_c r_c(n_c\vee m)\;+\;\frac{R^{2}\,(\log(n+m))^{2}\,r_{\rm tot}}{\kappa_{\min}^{2}}.
\end{align*}
Inserting into \eqref{eq:Delta-bound} and using $r_{\rm tot}\le\sum_c r_c(n_c\vee m)$ (since $n_c\vee m\ge 1$):
\[
  \frac{\|\mathbf\Delta\|_F^{2}}{nm}\;\lesssim\;\frac{\kappa_{\max}^{2}}{\kappa_{\min}^{3}}\cdot\frac{\log(n+m)}{N}\,\biggl[\sigma^{2}+\frac{R^{2}\,nm\,\log(n+m)}{N}\biggr]\sum_c r_c(n_c\vee m).
\]
This matches the bound after writing the overlap factor in the cleaner form $\kappa_{\max}^{2}/\kappa_{\min}^{3}$ (which equals $1$ in the disjoint case $\kappa\equiv 1$). The Case~A residual contributes terms of strictly smaller order in the matrix-completion regime, completing the proof.
\end{proof}

\subsubsection{Frobenius Error: GAME vs. Global Nuclear-Norm Regularization}\label{app:N--W-comparison}

We compare GAME with the standard nuclear-norm penalized estimator (N--W) on the same data: $\mathbf X = \mathbf W^\star + \mathbf E$ with $N$ uniform observations and subexponential noise. Under matching assumptions, N--W attains the per-entry rate $\sigma^{2} r^\star (n \vee m) \log(n+m)/N$ \citep[Cor.~1]{negahban2011restrictedstrongconvexityweighted} with $r^\star := \mathrm{rank}(\mathbf W^\star)$, while GAME attains $\sigma^{2} (\kappa_{\max}^{2}/\kappa_{\min}^{3}) \log(n+m) \sum_c r_c (n_c \vee m)/N$ by Theorem~\ref{thm:game}. The ratio of the two bounds is
\begin{equation}\label{eq:game-vs-N--W-ratio}
  \frac{\text{GAME bound}}{\text{N--W bound}}
  \;\asymp\;
  \frac{\kappa_{\max}^{2}}{\kappa_{\min}^{3}}
  \cdot
  \frac{\sum_{c} r_c\,(n_c \vee m)}{r^\star\, (n \vee m)},
\end{equation}
the product of an overlap penalty (which equals $1$ for disjoint covers) and a structural ratio that determines which estimator wins.

\paragraph{Bounded-overlap regime.} The comparison is sharpest under bounded
overlap with nontrivial block structure. Suppose $\kappa_{\max}\le\rho=O(1)$ or that $\kappa_{\text{max}}$ and $\kappa_{\text{min}}$ are comparable,
blocks are equal-sized ($n_c\asymp\rho n/|\mathcal C|$), each block has
common local rank $r_c\equiv r_{\rm loc}$, and we are in the tall regime
$n_c\gg m$ (so $n_c\vee m=n_c$). Then~\eqref{eq:main-bound-clean} gives GAME
complexity $\kappa_{\max}^2/\kappa_{\min}^3\sum_cr_c(n_c\vee m)\le\rho^3r_{\rm loc}n$.
Assuming block subspaces span $r^\star\gtrsim|\mathcal C|r_{\rm loc}/\rho$
distinct directions in $\mathbb R^m$, i.e., block subspaces are largely
distinct rather than rotations of a shared low-rank structure, 
\begin{equation}\label{eq:bounded-overlap-ratio}
\frac{\text{GAME bound}}{\text{N--W bound}}
\;\lesssim\;\frac{\rho^3r_{\rm loc}n}{(|\mathcal C|r_{\rm loc}/\rho)\,n}
=\frac{\rho^4}{|\mathcal C|},
\end{equation}
so GAME beats N--W by a factor $|\mathcal C|/\rho^4=|\mathcal C|\cdot O(1)$.

Imposing $\rho = O(1)$ constrains how many categories any single row
belongs to, not how many categories exist. If each row is tagged by a
constant number of attributes -- e.g., age bucket and location for an
experimental recording, giving $\rho = 2$ -- then refining either
attribute (finer age bins, more location codes) grows $|\mathcal C|$
without touching $\rho$. The regime $|\mathcal C| \gg \rho$ is
typical whenever rows are described by a few attributes, each
drawn from a large set of possible values.

\paragraph{Why N--W cannot recover this gap.} A natural objection is whether
a sharper analysis of N--W for global nuclear-norm regularization could match GAME's rate. It cannot, and the
obstruction is structural. Every matrix-completion bound needs three
ingredients to localize per block: spectral noise control, the cone
constraint, and RSC. All three are global for N--W because its regularizer
$\|\mathbf W\|_*$ has no notion of ``block.'' The dual constraint on $\lambda$
is therefore global, the cone is global, and the RSC inequality carries the
global dimension $(n+m)/2$. GAME's $\sum_c\lambda_c\|\mathbf W_c\|_*$
produces a per-block cone constraint (Lemma~\ref{lem:cone}); N--W's
$\|\mathbf W\|_*$ does not. Prior knowledge of block-low-rank structure
doesn't help N--W: the regularizer is structure-blind, and no proof
technique can change a regularizer's subdifferential. GAME's advantage pertains
to matching the regularizer to the structural assumption on
$\mathbf W^\star$.

\subsection{Subspace Recovery Results}

Many downstream applications care less about reconstructing $\mathbf W^\star$
entrywise than about recovering its per-category right-singular subspaces, i.e.
the latent directions along which each category varies. We translate the
Frobenius bound into such a per-block subspace bound via a
Yu--Wang--Samworth \citep{3fe21120-9dc1-3780-b0a4-37f58eabfcfa} variant of the Davis--Kahan/Wedin $\sin\Theta$
theorem.

\begin{corollary}[Per-block right-subspace recovery for GAME]\label{cor:subspace}
For each $c \in \mathcal C$, let $\mathbf Q_c^\star \in \mathbb R^{m \times r_c}$ denote the top $r_c$ right singular vectors of $\mathbf W_c^\star$, and let $\widehat{\mathbf Q}_c \in \mathbb R^{m \times r_c}$ denote the top $r_c$ right singular vectors of $\widehat{\mathbf W}_c$. Let $\sigma_{c,1}$ and $\sigma_{c, r_c}$ denote the largest and $r_c$-th singular values of $\mathbf W_c^\star$, respectively. Under the hypotheses of Theorem~\ref{thm:game}, on the same high-probability event the GAME estimator $\widehat{\mathbf W}$ satisfies, simultaneously for every $c \in \mathcal C$,
\begin{equation}\label{eq:per-block-subspace}
  \min_{R\in\mathrm{O}(r_c)} \bigl\|\widehat{\mathbf Q}_c R - \mathbf Q_c^\star\bigr\|_F^{2}
  \;\le\;
  a\,\frac{\bigl(2\sigma_{c,1} + \|\widehat{\mathbf W} - \mathbf W^\star\|_F\bigr)^{2}\,\|\widehat{\mathbf W} - \mathbf W^\star\|_F^{2}}{\sigma_{c, r_c}^{4}},
\end{equation}
where $\mathrm{O}(r_c)$ denotes the set of $r_c \times r_c$ orthogonal matrices and $a>0$ is a universal constant.
\end{corollary}

\begin{proof}
Fix $c \in \mathcal C$ and set $\mathbf\Delta_c := \widehat{\mathbf W}_c - \mathbf W_c^\star \in \mathbb R^{n_c \times m}$. Since $\mathbf\Delta_c$ is a row-restriction of the global error $\mathbf\Delta = \widehat{\mathbf W} - \mathbf W^\star$,
\begin{equation}\label{eq:row-restriction}
  \|\mathbf\Delta_c\|_F \;\le\; \|\mathbf\Delta\|_F \quad\text{and}\quad \|\mathbf\Delta_c\|_{op} \;\le\; \|\mathbf\Delta_c\|_F \;\le\; \|\mathbf\Delta\|_F.
\end{equation}

Write the thin SVDs $\mathbf W_c^\star = \mathbf U_c \mathbf\Sigma_c \mathbf V_c^\top$ and $\widehat{\mathbf W}_c = \widehat{\mathbf U}_c \widehat{\mathbf\Sigma}_c \widehat{\mathbf V}_c^\top$, with $\mathbf Q_c^\star$ comprising the leading $r_c$ columns of $\mathbf V_c$ and $\widehat{\mathbf Q}_c$ the leading $r_c$ columns of $\widehat{\mathbf V}_c$. By our assumption, $\mathrm{rank}(\mathbf W_c^\star) = r_c$, so $\sigma_{c, r_c+1} = 0$.

We invoke the Yu–Wang–Samworth variant of the Davis–Kahan/Wedin $\sin\Theta$ theorem for singular subspaces \citep[Theorem~3]{3fe21120-9dc1-3780-b0a4-37f58eabfcfa}, applied to $\mathbf W_c^\star$ and $\widehat{\mathbf W}_c$ with parameters $r = 1$, $s = r_c$, $d = r_c$. There exists an orthogonal matrix $\widehat O \in \mathrm{O}(r_c)$ such that
\begin{equation}\label{eq:yws}
  \bigl\|\widehat{\mathbf Q}_c \widehat O - \mathbf Q_c^\star\bigr\|_F
  \;\le\;
  \frac{2^{3/2}\,(2\sigma_{c,1} + \|\mathbf\Delta_c\|_{op})\,\min\bigl(r_c^{1/2}\|\mathbf\Delta_c\|_{op},\, \|\mathbf\Delta_c\|_F\bigr)}{\min(\sigma_{c,0}^{2} - \sigma_{c,1}^{2},\, \sigma_{c, r_c}^{2} - \sigma_{c, r_c+1}^{2})}.
\end{equation}
By the convention $\sigma_{c,0}^{2} = +\infty$, the first gap term in the denominator drops out. Since $\sigma_{c, r_c+1} = 0$, the denominator equals $\sigma_{c, r_c}^{2}$.

Bound the numerator using \eqref{eq:row-restriction}: $\|\mathbf\Delta_c\|_{op} \le \|\mathbf\Delta_c\|_F \le \|\mathbf\Delta\|_F$, and $\min(r_c^{1/2}\|\mathbf\Delta_c\|_{op}, \|\mathbf\Delta_c\|_F) \le \|\mathbf\Delta_c\|_F \le \|\mathbf\Delta\|_F$. So \eqref{eq:yws} gives
\[
  \bigl\|\widehat{\mathbf Q}_c \widehat O - \mathbf Q_c^\star\bigr\|_F
  \;\le\;
  \frac{2^{3/2}\,(2\sigma_{c,1} + \|\mathbf\Delta\|_F)\,\|\mathbf\Delta\|_F}{\sigma_{c, r_c}^{2}}.
\]
Squaring,
\[
  \bigl\|\widehat{\mathbf Q}_c \widehat O - \mathbf Q_c^\star\bigr\|_F^{2}
  \;\le\;
  \frac{8\,(2\sigma_{c,1} + \|\mathbf\Delta\|_F)^{2}\,\|\mathbf\Delta\|_F^{2}}{\sigma_{c, r_c}^{4}}.
\]
By Procrustes optimality, $\min_{R \in \mathrm{O}(r_c)} \|\widehat{\mathbf Q}_c R - \mathbf Q_c^\star\|_F^{2} \le \|\widehat{\mathbf Q}_c \widehat O - \mathbf Q_c^\star\|_F^{2}$, establishing \eqref{eq:per-block-subspace} with $a = 8$. The simultaneous statement over $c \in \mathcal C$ follows because the bound conditions on the Theorem~\ref{thm:game} event, which holds simultaneously for all $c$.
\end{proof}

The N--W analog of~\eqref{eq:per-block-subspace} carries the same form with
$\|\widehat{\mathbf W}^{\mathrm{NW}}-\mathbf W^\star\|_F$ in the numerator, where
the conditioning factors $\sigma_{c,1},\sigma_{c,r_c}$ depend only on
$\mathbf W_c^\star$ and cancel in any ratio. So the ratio of subspace errors
reduces to the ratio of Frobenius errors, and the bounded-overlap analysis
of Section~\ref{app:N--W-comparison} applies: GAME beats N--W on per-block
subspace recovery by factor $|\mathcal C|/\rho^4$ in the
regime~\eqref{eq:bounded-overlap-ratio}.

\subsection{Reading the Theoretical Results}\label{app:theory-discussion}

The theoretical results above clarify how group-aware regularization changes the statistical
problem from a single global matrix-completion problem into a collection of
coupled category-wise estimation problems. This is important because the central
modeling premise of GAME is not merely that the full matrix is low rank, but
that each meta-category may have its own lower-dimensional latent structure.
The theory makes this premise explicit: the dominant complexity term in
Theorem~\ref{thm:game} is
\[
  \sum_{c\in\mathcal C} r_c\,(n_c\vee m),
\]
rather than the global nuclear-norm complexity
\[
  r^\star\,(n\vee m).
\]
Thus, GAME is theoretically favorable in regimes where the matrix is only
moderately low rank globally, but substantially lower rank within meaningful
categories. In such settings, the estimator pays for the intrinsic complexity
of each category rather than the rank of the entire heterogeneous population.

A key practical consequence of the theory is that it gives a principled
category-wise interpretation of the regularization parameters
\(\lambda_c\). The per-group noise spectral bound shows that \(\lambda_c\)
should dominate the stochastic fluctuation of the observed noise restricted to
category \(c\):
\[
  \lambda_c \;\gtrsim\;
  \frac{\sigma}{\kappa_{\min}}
  \sqrt{\frac{N_c\log d_c}{\min(n_c,m)}}
  +
  \frac{R\log d_c}{\kappa_{\min}}.
\]
Equivalently, in the dominant sub-Exponential matrix-completion regime, the
appropriate regularization level scales with the amount of noise visible inside
category \(c\), the effective number of observed entries \(N_c\), and the local
block dimension \(d_c=n_c+m\). This gives a concrete justification for choosing
different penalties across meta-categories. Smaller or more sparsely observed
categories should generally receive stronger relative regularization, while
larger and better-sampled categories can support weaker shrinkage. In this
sense, the theorem converts \(\lambda_c\) from a purely empirical tuning
parameter into a quantity that can be calibrated from the category-wise sample
size and noise level.

The role of \(N_c\) is equally important. Although the observations are drawn
globally, the analysis shows that each category receives an effective sample
budget
\[
  N_c \approx N\,\frac{n_c}{n}
\]
under uniform sampling. The concentration result for \(N_c\) then ensures that,
with high probability, each category receives roughly its expected number of
samples. This allows the global sampling process to be analyzed through
category-wise restricted strong convexity conditions. As a result, the theorem
identifies the sample size needed not only for the entire matrix, but for every
category to be statistically identifiable. In particular, categories with small
\(n_c\) require enough observations so that \(N_c\) is large relative to the
local block dimension and rank. This makes explicit a limitation of group-aware
completion: GAME can exploit category structure only when each relevant
category is sufficiently sampled.

The overlap factors \(\kappa_{\max}\) and \(\kappa_{\min}\) quantify the cost
of allowing rows to belong to multiple meta-categories. When the categories are
disjoint, this factor is equal to one and GAME behaves like a block-aware
matrix-completion procedure. When categories overlap, the proof must account
for the fact that the same row contributes to multiple nuclear-norm penalties
and multiple category-wise loss decompositions. The resulting factor
\[
  \frac{\kappa_{\max}^2}{\kappa_{\min}^3}
\]
measures the statistical price of overlap. Importantly, this price is controlled
when each row belongs to only a bounded number of categories. Therefore, GAME is
especially well suited to settings where rows are annotated by a small number
of meaningful attributes, such as age group, occupation, region, recording
session, or species label, even if the total number of possible categories is
large.

The comparison with global nuclear-norm regularization explains why these
bounds are not merely a restatement of standard matrix-completion theory. A
global nuclear-norm estimator uses a single regularizer, a single dual-noise
condition, and a single global cone constraint. Consequently, its rate depends
on the global rank \(r^\star\). GAME instead induces a collection of local cone
constraints through the regularizer
\[
  \sum_{c\in\mathcal C}\lambda_c\|\mathbf W_c\|_*,
\]
which is what allows the proof to localize the statistical complexity to the
category level. This is the main theoretical reason GAME can outperform global
matrix completion when the data are heterogeneous: it regularizes according to
the latent structure present within each category rather than forcing all rows
to share one global low-rank representation.

Finally, the subspace recovery result shows that the Frobenius error bound has
implications beyond entrywise reconstruction. Many downstream tasks, such as
clustering, classification, or representation learning, depend on recovering
the latent directions associated with each category. Corollary~\ref{cor:subspace}
shows that if the Frobenius error is small and the category-specific singular
gap is sufficiently large, then GAME also recovers the right singular subspace
of each block. Thus, the theory supports the broader motivation for GAME: the
method is not only designed to impute missing entries, but also to preserve the
category-specific latent geometry that downstream analyses rely on.

\newpage
\section{Experiments}\label{appendix:experiments}

\begin{figure}[ht]
    \centering
    \includegraphics[width=.6\linewidth]{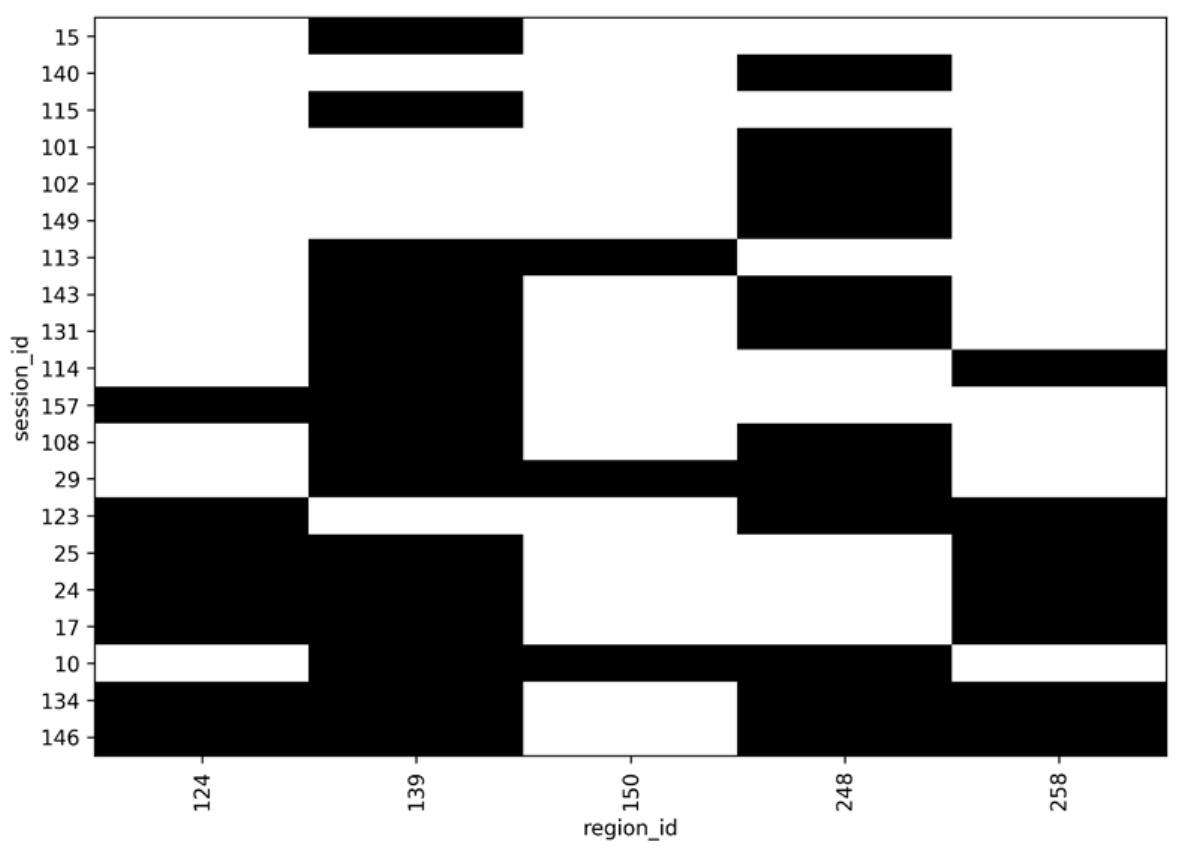}
    \caption{\textbf{Structured missingness of regions across Neuropixel recording sessions from \citep{CHEN2024676}.} White shows observed region-session pairs and black shows structured missingness as a result of experimental design and technological limitations.}
    \label{fig:session_region_heatmap}
\end{figure}
\begin{figure}[ht!]
    \centering
    \includegraphics[width=0.8\linewidth]{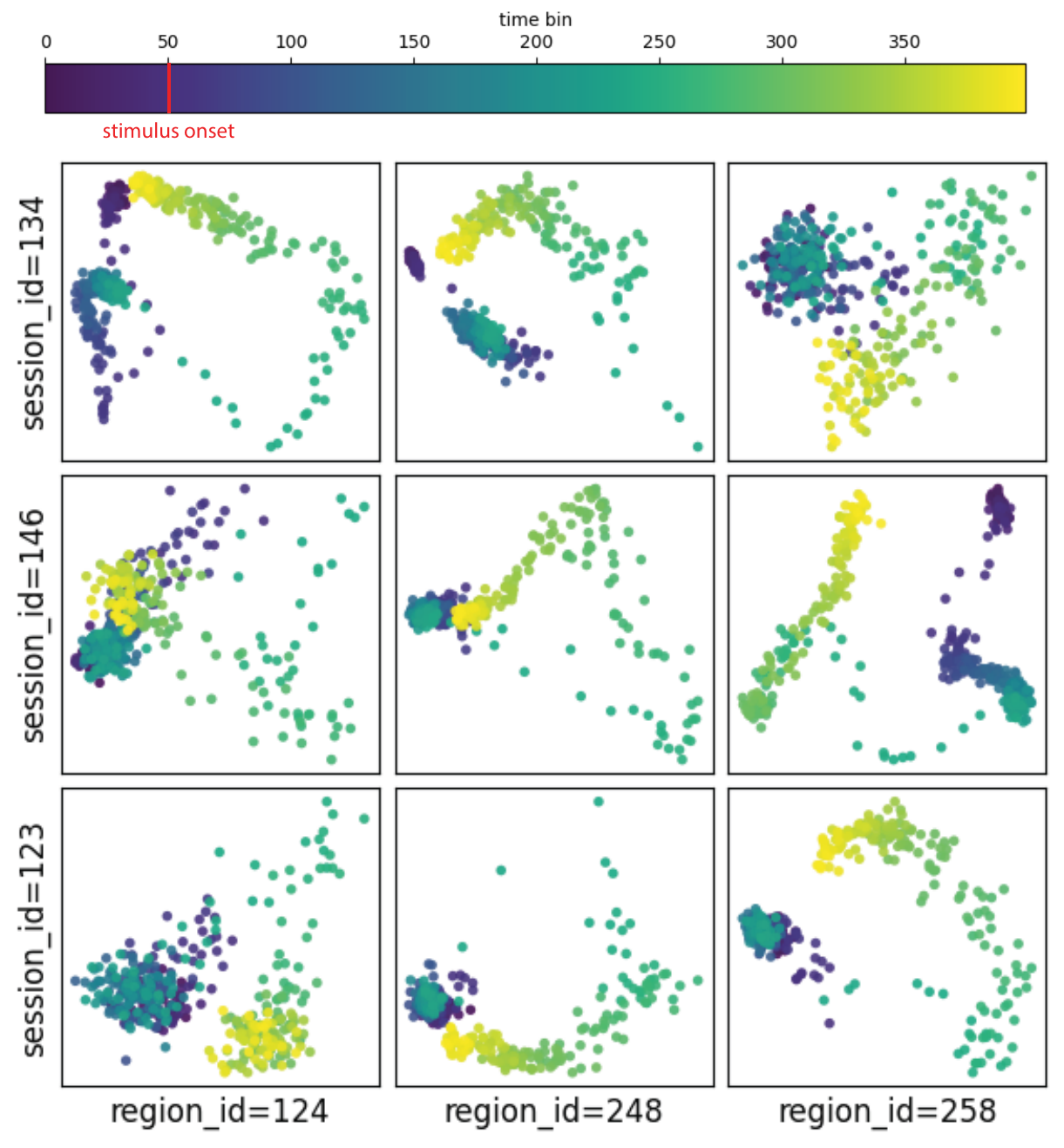}
    \caption{\textbf{Principal dynamics by subsetted session-region combinations from \citep{CHEN2024676}}. Region 124 corresponds to midbrain reticular nucleus, 248 corresponds to striatum, and 258 corresponds to the superior colliculus. }
    \label{fig:session_by_region_pca}
\end{figure}
\begin{figure}[H]
    \centering
    \includegraphics[width=1\linewidth]{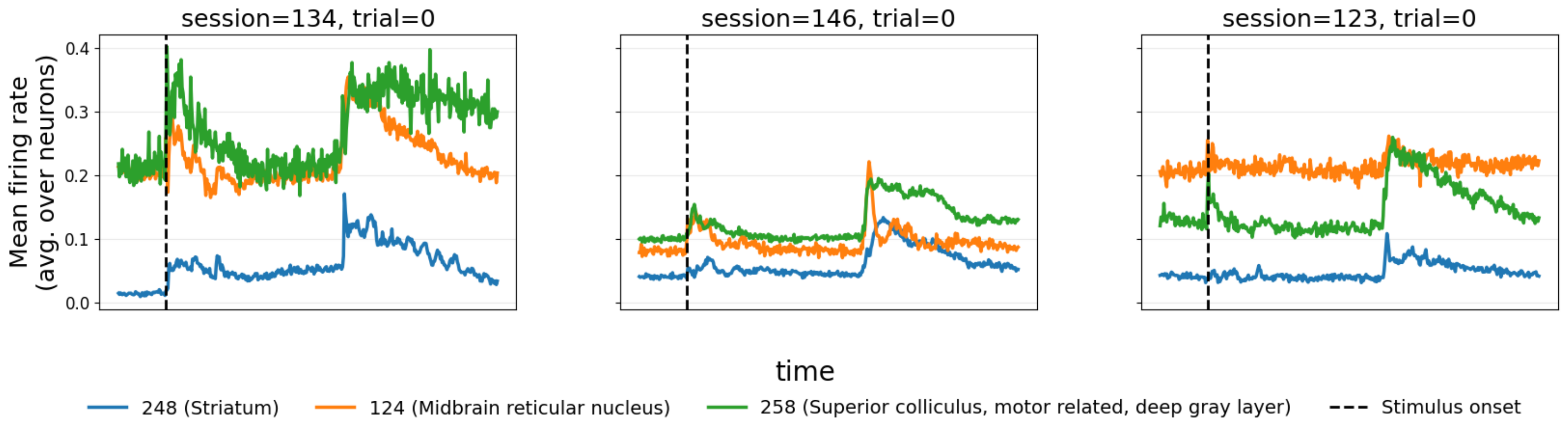}
    \caption{\textbf{Mean firing rates of three brain regions over three recording sessions from \citep{CHEN2024676}}. Neuron spike trains are averaged to get mean firing rates for each brain region over time. Though single regions exhibit similar temporal dynamics across the sessions, firing rate comparisons between regions are visibly inconsistent.}
    \label{fig:firingrates}
\end{figure}

\subsection{Experiments Compute Resources}

Experiments on synthetic, MovieLens, and BirdSet datasets were executed locally on a 2024 Apple Silicon MacBook environment. Due to the larger scale of the Svoboda Lab Neuropixels dataset, those experiments were executed on a university high-performance computing cluster using GPU-enabled compute nodes. Typical Neuropixels runs required approximately one hour for 500 GAME iterations.


\end{document}